\documentclass[fleqn,12pt]{article} 
\pdfoutput=1
\usepackage{a4wide}
\usepackage{mathtools}
\usepackage{geometry}
\usepackage{graphicx}
\usepackage{amssymb}
\usepackage{amsmath}
\usepackage{amsthm}
\usepackage{color}
\usepackage{bm}
\usepackage{extarrows}
\usepackage{chngcntr}
\usepackage{enumitem}
\usepackage{multirow}
\usepackage{makecell}
\usepackage{breakcites}
\usepackage[utf8]{inputenc}
\usepackage[english]{babel}
\usepackage[hidelinks]{hyperref}
\usepackage{cleveref}
\usepackage{dsfont}
\usepackage{calc}

\usepackage{xifthen}
\usepackage{xspace}

\usepackage{natbib}

\usepackage{etoolbox}
\makeatletter
\newcommand\bibstyle@comma{\bibpunct(),a,,}
\newcommand\bibstyle@semicolon{\bibpunct();a,,}
\makeatother
\pretocmd\citet{\citestyle{comma}}\relax\relax
\pretocmd\Citet{\citestyle{comma}}\relax\relax
\pretocmd\citep{\citestyle{semicolon}}\relax\relax
\pretocmd\Citep{\citestyle{semicolon}}\relax\relax

\newtheorem{mythm}{Theorem}
\numberwithin{mythm}{section}
\newtheorem{mylem}[mythm]{Lemma}

\newtheorem{mydef}[mythm]{Definition}
\newtheorem{mybsp}[mythm]{Example}
\newtheorem{myann}[mythm]{Assumption}

\newtheorem{mybem}[mythm]{Remark}

\numberwithin{figure}{section}

\crefname{mythm}{theorem}{theorems}
\crefname{mylem}{lemma}{lemmas}
\crefname{myprop}{proposition}{propositions}
\crefname{myann}{assumption}{assumptions}

\newcommand{\iid}{\text{i.i.d.}}
\newcommand{\ie}{\enm{\text{i.e.\@}}}
\newcommand{\vgl}{\enm{\text{cf.\@}}}
\newcommand{\seite}{\enm{\text{p.\@}}}

\newcommand{\eg}{\enm{\text{e.g.\@}}}
\newcommand{\wenn}{\text{, if }}
\newcommand{\sonst}{\text{, else}}
\newcommand{\fs}{\enm{\text{a.s.\@}}}
\newcommand{\einschraenkung}{\,\rule[-2mm]{0.1mm}{5mm}\,{}}
\newcommand{\einschraenkungklein}{\,\rule[-1mm]{0.1mm}{3mm}\,{}}
\newcommand{\enm}[1]{\ensuremath{#1}\xspace}
\newcommand{\limn}{\enm{\lim_{n\to\infty}}}
\newcommand{\diff}{\enm{\,\mathrm{d}}}
\newcommand{\R}{\enm{\mathbb{R}}}
\newcommand{\N}{\enm{\mathbb{N}}}
\newcommand{\eps}{\enm{\varepsilon}}
\newcommand{\lb}{\enm{\lambda}}

\newcommand{\Ind}[1][]{\enm{\mathds{1}_{#1}}}
\renewcommand{\P}{\enm{\textnormal{P}}}
\newcommand{\Pxvar}{\enm{\P^X}}
\newcommand{\Pbed}[2][\cdot]{\enm{\P(#1\,|\,#2)}} 
\newcommand{\Q}{\enm{\textnormal{Q}}}
\newcommand{\D}{\enm{D}} 
\newcommand{\DVert}{\enm{\textnormal{D}}}
\newcommand{\Dn}{\enm{\D_n}} 
\newcommand{\DVertn}{\enm{\DVert_n}}
\newcommand{\ew}[2][]{\enm{\mathbb{E}_{#1}\left[#2\right]}}
\newcommand{\X}{\enm{\mathcal{X}}}
\newcommand{\Y}{\enm{\mathcal{Y}}}
\newcommand{\XX}{\enm{\X\times\X}}
\newcommand{\XY}{\enm{\X\times\Y}}
\newcommand{\XYR}{\enm{\X\times\Y\times\R}}
\newcommand{\YR}{\enm{\Y\times\R}}
\newcommand{\MXY}{\enm{\mathcal{M}_1(\XY)}}
\newcommand{\BX}{\enm{\mathcal{B}_{\X}}}
\newcommand{\BY}{\enm{\mathcal{B}_{\Y}}}
\newcommand{\BXY}{\enm{\mathcal{B}_{\XY}}}
\renewcommand{\L}[2][]{\enm{L_{#2}\ifthenelse{\isempty{#1}}{}{(#1)}}}
\newcommand{\Lp}{\L{p}}
\newcommand{\Linfty}{\L{\infty}}
\newcommand{\Lppxvar}{\enm{\Lp(\Pxvar)}} 
\newcommand{\Linftypxvar}{\enm{\Linfty(\Pxvar)}}
\newcommand{\Linftypnix}[1][i]{\enm{\Linfty(\Pnix[#1])}}

\newcommand{\norm}[2]{\enm{\left|\left|#2\right|\right|_{#1}}}
\newcommand{\normSup}[1]{\norm{\infty}{#1}}
\newcommand{\normH}[1]{\norm{\H}{#1}}
\newcommand{\normLppxvar}[1]{\norm{\Lppxvar}{#1}} 
\newcommand{\normLinftypxvar}[1]{\norm{\Linftypxvar}{#1}}
\newcommand{\normLinftypnix}[2][i]{\norm{\Linftypnix[#1]}{#2}}
\renewcommand{\H}{\enm{H}} 
\renewcommand{\k}{\enm{k}}
\newcommand{\loss}{\enm{L}}

\newcommand{\risk}[1][\loss,\P]{\enm{\mathcal{R}_{#1}}} 
\newcommand{\riskempn}[1][\loss,\DVertn]{\enm{\mathcal{R}_{#1}}} 
\newcommand{\riskbayes}[1][\loss,\P]{\enm{\mathcal{R}_{#1}^*}}
\newcommand{\riskoptH}[1][\loss,\P,\H]{\enm{\mathcal{R}_{#1}^*}} 
\newcommand{\innerrisk}[1][\loss,\Q]{\enm{\mathcal{C}_{#1}}} 
\newcommand{\innerriskbed}[1][x]{\innerrisk[\loss,\Pbed{#1}]} 
\newcommand{\innerriskbedbayes}[1][x]{\enm{\innerrisk[\loss,\Pbed{#1}]^*}} 
\newcommand{\fbayes}{\enm{f_{\loss,\P}^*}}

\newcommand{\fn}{\enm{f_n}}

\newcommand{\ftilde}[1][]{\ifthenelse{\isempty{#1}}{\enm{\tilde{f}}}{\enm{\tilde{f}_{#1}}}}
\newcommand{\unif}[1][\enm{0,1}]{\enm{\mathcal{U}(#1)}}

\newcommand{\fntheo}{\enm{f_{\P,n}}} 
\newcommand{\fnemp}{\enm{f_{\DVertn,n}}} 
\newcommand{\fntheoext}{\enm{f_{\loss,\P,\lbnbold,\knbold}}} 
\newcommand{\fnempext}{\enm{f_{\loss,\DVertn,\lbnbold,\knbold}}} 
\newcommand{\fnitheo}[1][i]{\enm{f_{\P,n,#1}}} 
\newcommand{\fniemp}[1][i]{\enm{f_{\DVertn,n,#1}}} 
\newcommand{\fnitheodach}[1][i]{\enm{\hat{f}_{\P,n,#1}}} 
\newcommand{\fniempdach}[1][i]{\enm{\hat{f}_{\DVertn,n,#1}}} 
\newcommand{\fnitheoext}[1][i]{\enm{f_{\loss,\Pni[#1],\lbni[#1],\kni[#1]}}} 
\newcommand{\fniempext}[1][i]{\enm{f_{\loss,\DVertni[#1],\lbni[#1],\kni[#1]}}} 
\newcommand{\fnitheoextdach}[1][i]{\enm{\hat{f}_{\loss,\Pni[#1],\lbni[#1],\kni[#1]}}} 
\newcommand{\fniempextdach}[1][i]{\enm{\hat{f}_{\loss,\DVertni[#1],\lbni[#1],\kni[#1]}}} 
\newcommand{\fnitildenull}[1][i]{\enm{f_{\loss,\P,\lbntilde,\kninulltilde[#1]}}}
\newcommand{\fnjnull}{\enm{f_{\loss,\P,\lbntilde,\kjr[0]}}}
\newcommand{\fninull}[1][i]{\enm{f_{\loss,\Pni[#1],\betani[#1]^2\lbni[#1],\kninull[#1]}}}

\makeatletter
\def\moverlay{\mathpalette\mov@rlay}
\def\mov@rlay#1#2{\leavevmode\vtop{%
		\baselineskip\z@skip \lineskiplimit-\maxdimen
		\ialign{\hfil$\m@th#1##$\hfil\cr#2\crcr}}}
\newcommand{\charfusion}[3][\mathord]{
	#1{\ifx#1\mathop\vphantom{#2}\fi
		\mathpalette\mov@rlay{#2\cr#3}
	}
	\ifx#1\mathop\expandafter\displaylimits\fi}
\makeatother

\newcommand{\kbold}{\enm{\bm{k}}} 

\newcommand{\smax}{\enm{s_\text{max}}}
\newcommand{\Xnbold}{\enm{\bm{\mathcal{X}_n}}}
\newcommand{\IndexPpositive}[1][\Xnbold,\P]{\enm{I_{#1}}}
\newcommand{\AnzPpositive}{\enm{\tilde{m}_n}} 
\newcommand{\Ij}[1][j]{\enm{I^{(#1)}}} 
\newcommand{\Xni}[1][i]{\enm{\X_{n,#1}}}

\newcommand{\Pni}[1][i]{\enm{\P_{n,#1}}}
\newcommand{\Pnix}[1][i]{\enm{\P_{n,#1}^X}}
\newcommand{\Dni}[1][i]{\enm{\D_{n,#1}}} 
\newcommand{\dni}[1][i]{\enm{d_{n,#1}}} 
\newcommand{\DVertni}[1][i]{\enm{\DVert_{n,#1}}}
\newcommand{\wni}[1][i]{\enm{w_{n,#1}}}
\newcommand{\lbnbold}{\enm{\bm{\lambda_{n}}}} 
\newcommand{\lbni}[1][i]{\enm{\lb_{n,#1}}}
\newcommand{\lbntilde}{\enm{\tilde{\lb}_n}} 
\newcommand{\knbold}{\enm{\bm{k_{n}}}} 
\newcommand{\kni}[1][i]{\enm{\k_{n,#1}}}
\newcommand{\knitilde}[1][i]{\enm{\tilde{\k}_{n,#1}}} 
\newcommand{\kninull}[1][i]{\enm{\k_{n,#1}^{(0)}}}
\newcommand{\kninulltilde}[1][i]{\enm{\tilde{\k}_{n,#1}^{(0)}}}
\newcommand{\kjr}[1][r]{\enm{\k^{(j,#1)}}} 
\newcommand{\kjbold}[1][j]{\enm{\bm{k^{(#1)}}}}
\newcommand{\betabold}{\enm{\bm{\beta}}}
\newcommand{\betajr}[1][r]{\enm{\beta^{(j,#1)}}}
\newcommand{\betajbold}[1][j]{\enm{\bm{\beta^{(#1)}}}}
\newcommand{\betani}[1][i]{\enm{\beta_{n,#1}}}
\newcommand{\Hni}[1][i]{\enm{\H_{n,#1}}}
\newcommand{\Hnitilde}[1][i]{\enm{\tilde{\H}_{n,#1}}} 
\newcommand{\Hninull}[1][i]{\enm{\H_{n,#1}^{(0)}}}
\newcommand{\Hninulltilde}[1][i]{\enm{\tilde{\H}_{n,#1}^{(0)}}}
\newcommand{\Hjr}[1][r]{\enm{\H^{(j,#1)}}} 

\title{\textbf{Lp- and Risk Consistency of Localized SVMs}}
\date{May 16, 2023}
\author{\textbf{Hannes K\"ohler}\thanks{Email: \href{mailto:hannes.koehler@uni-bayreuth.de}{\texttt{hannes.koehler@uni-bayreuth.de}}}\\
		Department of Mathematics, University of Bayreuth, Germany
	}

\begin{document}
\maketitle

\begin{abstract}
	Kernel-based regularized risk minimizers, also called support vector machines (SVMs), are known to possess many desirable properties but suffer from their super-linear computational requirements when dealing with large data sets. This problem can be tackled by using localized SVMs instead, which also offer the additional advantage of being able to apply different hyperparameters to different regions of the input space. In this paper, localized SVMs are analyzed with regards to their consistency. It is proven that they inherit \Lp- as well as risk consistency from global SVMs under very weak conditions and even if the regions underlying the localized SVMs are allowed to change as the size of the training data set increases.
	
	\vspace*{1ex}\noindent\textbf{Keywords:} localized learning, consistency, kernel methods, support vector machines, big data
\end{abstract}

\section{Introduction}\label{Sec:LocCons_Intro}

Kernel-based regularized risk minimizers based on a general loss function, which are also known as (general) support vector machines (SVMs), play an important role in statistical machine learning, which is due to two main reasons: First, they are known to possess many desirable theoretical properties such as universal consistency, statistical robustness and stability, and good learning rates, \vgl \citet{vapnik1995,vapnik1998,schoelkopf2002,cucker2007,steinwart2008}. Secondly, they are the solutions of finite-dimensional convex programs \citep[\vgl][]{smola2004} and empirically observe good performance \citep[\vgl][]{klambauer2017,paoletti2019}---at least if the data set is not too large. For large data sets, SVMs however suffer from their computational requirements growing at least quadratically in the number of training samples, with regards to both time and memory, \vgl \citet{platt1998,joachims1998,thomann2017}.

There exist different approaches to circumvent this problem, one of them being the use of localized SVMs, which implement the idea of not computing one SVM on the whole input space but instead dividing this input space into different (not necessarily disjoint) regions, computing SVMs on each of these regions, and then joining them together in order to obtain a global predictor. In addition to the computational advantage this approach offers, it can also yield improved predictions as it adds flexibility by allowing for differing underlying hyperparameters being chosen in the different regions. In \Cref{SubSec:LocCons_Loc_Overview}, we discuss these advantages in more detail, as well as briefly mentioning some of the different approaches for circumventing the computational challenges. 

The main goal of this paper is to derive new theoretical results on such localized SVMs. More specifically, we prove that localized SVMs are risk consistent as well as \Lp-consistent under certain mild conditions. Notably, we also allow for the regionalization, which underlies a localized SVM, to change as the size of the data set increases. Because of SVMs being defined as minimizers of some regularized risk function, risk consistency is the natural type of consistency to consider, and there already exist some results on risk consistency respectively learning rates (which imply risk consistency) of localized SVMs, \vgl \citet{hable2013,meister2016,dumpert2018,blaschzyk2022} among others. However, all of these in some aspects offer considerably less generality than the result we derive. On the other hand, \Lp-consistency is of interest as it compares functions themselves instead of their risks, and, to our knowledge, there do not exist any results on \Lp-consistency of localized SVMs so far.

The paper is organized as follows: \Cref{Sec:LocCons_Pre} contains some general prerequisites as well as a formal definition of SVMs, whereas the localized approach is described in more detail in \Cref{Sec:LocCons_Loc}. The main results can be found in \Cref{Sec:LocCons_Cons}, and finally, \Cref{Sec:LocCons_Discussion} gives a short summary.

\section{Prerequisites}\label{Sec:LocCons_Pre}

Before introducing localized SVMs in \Cref{SubSec:LocCons_Loc_Pre} and stating our results about their consistency in \Cref{Sec:LocCons_Cons}, we first need to define the underlying (non-localized) SVMs in more detail as well as state some additional prerequisites.

Given a training data set $\Dn:=((x_1,y_1),\dots,(x_n,y_n))\in(\XY)^n$ consisting of independent and identically distributed (\iid) observations sampled from some unknown probability measure \P on a space \XY, we aim at learning a function $f\colon\X\to\R$. More specifically, we denote by $(X,Y)$ a pair of random variables with values in \XY distributed according to \P, and the goal is to estimate certain characteristics of the conditional distribution \Pbed{X} of $Y$ given $X$. We impose the following standard and not very restrictive assumptions on the underlying space \XY throughout this paper:

\begin{myann}\label{Ann:LocCons_Pre_AllgAnn}
	Let \X be a complete separable metric space and let $\Y\subseteq\R$ be closed. Let \X and \Y be equipped with their respective Borel $\sigma$-algebras \BX and \BY. Let $\P\in\MXY$, where \MXY denotes the set of all Borel probability measures on the measurable space $(\XY,\BXY)$.
\end{myann}

Notably, $\Y\subseteq\R$ guarantees that the conditional probability \Pbed{X} does indeed uniquely exist \citep[\vgl][Theorems 10.2.1 and 10.2.2]{dudley2004} because \Y is Polish \citep[\vgl][\seite 157]{bauer2001}.

Which exact characteristics of \Pbed{X} are to be learned is determined by the chosen \textit{loss function}, which is a measurable function $\loss\colon\XYR\to[0,\infty)$. For example, estimating the conditional mean function can be approached by using the least squares loss, and conditional quantile functions can be estimated by using the pinball loss. $\loss(x,y,f(x))$ quantifies the loss associated with predicting $f(x)$ while the true output belonging to $x$ is $y$, and the goal is to find a predictor whose expected loss is as small as possible. To this end, we call
\begin{align*}
	\risk(f) := \ew[\P]{\loss(X,Y,f(X))}
\end{align*}
\textit{\loss-risk} (or just \textit{risk}) of a measurable function $f$, and
\begin{align*}
	\riskbayes := \inf\{\risk(f) \,|\, f\colon\X\to\R \text{ measurable}\}\,
\end{align*}
\textit{Bayes risk}. We call a measurable function \fbayes achieving $\risk(\fbayes)=\riskbayes$ a \textit{Bayes function}.

A sequence $(\fn)_{n\in\N}$ is called \textit{risk consistent} if
\begin{align*}
	\risk(\fn) \to \riskbayes\,,\qquad n\to\infty\,,
\end{align*}
in probability, and it is called \textit{\Lp-consistent} for some $p\in[1,\infty)$ if
\begin{align*}
	\normLppxvar{\fn-\fbayes} \to 0\,,\qquad n\to\infty\,,
\end{align*}
in probability, where \Pxvar denotes the marginal distribution on \X associated with \P. For the latter consistency property, we always assume \fbayes to \Pxvar-almost surely (\fs) uniquely exist. As mentioned in the introduction, the notion of \Lp-consistency does directly depend on the difference between the functions instead of on the difference between their risks, which additionally depends on the loss function and the conditional distribution of $Y$.

As \P is unknown, it is not possible to minimize \risk directly and one instead has to use the \textit{empirical risk}
\begin{align*}
	\riskempn(f) := \ew[\DVertn]{\loss(X,Y,f(X))} = \frac{1}{n}  \sum_{i=1}^{n} \loss(x_i,y_i,f(x_i))\,,
\end{align*}
where 
\begin{align*}
	\DVertn := \frac{1}{n}  \sum_{i=1}^{n} \delta_{(x_i,y_i)}\,
\end{align*}
is the empirical distribution corresponding to \Dn, with $\delta_{(x_i,y_i)}$ denoting the Dirac measure in $(x_i,y_i)$. In order to avoid overfitting, a regularization term  is added to this empirical risk, which results in the \textit{empirical SVM} being defined as the solution of the minimization problem
\begin{align}\label{eq:LocCons_Pre_empSVM}
	f_{\loss,\DVertn,\lb,\k} := \arg\inf_{f\in\H} \riskempn(f) + \lb\normH{f}^2\,.
\end{align}
Here, $\lb>0$ controls the amount of regularization and \H is a \textit{reproducing kernel Hilbert space} (RKHS) over \X. Each such RKHS is associated with a \textit{kernel} on \X, which is a symmetric and positive definite function $\k\colon\XX\to\R$. We call \k bounded if $\normSup{\k}:=\sup_{x\in\X}\sqrt{\k(x,x)}<\infty$. We refer to \citet{aronszajn1950,berlinet2004,saitoh2016} for a detailed introduction of kernels, RKHSs and their properties. 

The goal of \Cref{Sec:LocCons_Cons} is to derive \Lp- respectively risk consistency of localized versions of such SVMs as the size $n$ of the data set increases. As an intermediate step in the according proofs, we additionally need the \textit{theoretical SVM}
\begin{align}\label{eq:LocCons_Pre_theoSVM}
	f_{\loss,\P,\lb,\k} := \arg\inf_{f\in\H} \risk(f) + \lb\normH{f}^2\,.
\end{align}

As a last part of these prerequisites, we need to specify some properties of loss functions. We only investigate loss functions which are convex---by which we mean convexity in the last argument of \loss---and additionally distance-based. The latter is a property that is satisfied by most of the typical loss functions for regression tasks, but not necessarily by those used in classification tasks. However, some distance-based losses are also popular choices in classification tasks, like for example the least squares loss, \vgl \citet[Section~1.4]{gyorfi2002}.

\begin{mydef}\label{Def:LocCons_Pre_DistBasedLoss}
	A loss function $\loss\colon\XYR\to[0,\infty)$ is called \textit{distance-based} if there exists a representing function $\psi\colon \R\to[0,\infty)$ satisfying $\psi(0)=0$ and $\loss(x,y,t)=\psi(y-t)$ for all $(x,y,t)\in\XYR$. 
	\\
	Let $p\in(0,\infty)$. A distance-based loss $\loss\colon\XYR\to[0,\infty)$ with representing function $\psi$ is of
	\begin{enumerate}[label=(\roman*)]
		\item \textit{upper growth type} $p$ if there is a constant $c>0$ such that
		\begin{align*}
			&\psi(r) \le c\, (|r|^p+1) &\forall\, r\in\R\,,
			\intertext{\item \textit{lower growth type} $p$ if there is a constant $c>0$ such that}
			&\psi(r) \ge c\,|r|^p - 1 &\forall\, r\in\R\,,
		\end{align*}
		\item \textit{growth type} $p$ if \loss is of both upper and lower growth type $p$.
	\end{enumerate}
\end{mydef}
Since the first argument does not matter in distance-based loss functions, we often ignore it and write $\loss\colon\YR\to[0,\infty)$ and $\loss(y,t)$ instead.

For example, the aforementioned least squares loss and pinball loss are of growth type 2 and 1 respectively. Depending on the growth type $p$, our results require that the \textit{averaged $p$-th moment} of \P is finite, which guarantees that there exists a function in \H that has finite risk. This averaged $p$-th moment is defined as
\begin{align*}
	|\P|_p := 
	\left(\int_{\X}\int_{Y} |y|^p \diff\Pbed[y]{x}\diff\Pxvar(x)\right)^{1/p} = \left(\int_{\X}|\Pbed{x}|_p^p\diff\Pxvar(x)\right)^{1/p}\,,
\end{align*}
thus making the moment condition $|\P|_p<\infty$ slightly more restrictive when dealing with loss functions of a higher growth type. In the definition of the averaged $p$-th moment, $|\Pbed{x}|_p$ denotes the \textit{$p$-th moment} of \Pbed{x}, where for an arbitrary distribution \Q on \Y this $p$-th moment is defined by
\begin{align*}
	|\Q|_p := \left(\int_{\Y} |y|^p \diff\Q(y)\right)^{1/p}\,.
\end{align*}

\section{Localized Approach}\label{Sec:LocCons_Loc}

As mentioned in the introduction, SVMs, while possessing many desirable theoretical properties, suffer from their super-linear (with respect to the size of the training data set) computational requirements when dealing with large data sets. There exist different approaches to reduce this computational complexity, one of them being localization. \Cref{SubSec:LocCons_Loc_Overview} gives a quick overview of some existing approaches as well as an introduction of the idea behind and the additional advantages of the localization approach. \Cref{SubSec:LocCons_Loc_Pre} formally defines localized SVMs and states requirements which the underlying structure, like the regions and the applied kernels, need to satisfy.

\subsection{Overview of localized and other approaches}\label{SubSec:LocCons_Loc_Overview}

Approaches to reduce the computational complexity of SVMs include online learning approaches such as stochastic gradient descent \citep[\eg,][]{smale2006,ying2006,dieuleveut2016,lin2016,lin2017b} as well as algorithms approximating the kernel matrix via column subsampling \citep[\eg,][]{williams2001,bach2013,alaoui2015,rudi2015} and random feature approximations of the kernel \citep[\eg,][]{rahimi2008,sriperumbudur2015,rudi2017,liu2022,mei2022}, with \citet{yang2012} comparing the last two approaches. Additionally, there are also methods combining multiple of these approaches \citep[\eg,][]{rudi2017b,meanti2020}.

Closer to the localized approach are methods that decompose the available data set into $m\in\N$ subsets and train $m$ ``small'' SVMs on these subsets instead of a single ``large'' one on all of \Dn, which can substantially reduce the training time as well as required storage space because of the aforementioned super-linear computational requirements of SVMs. This can for example be done by means of distributed learning \citep[\eg,][]{christmann2007b,zhang2015,guo2017,lin2017,muecke2018,lin2020}, which randomly splits \Dn into subsets, trains an SVM on each such subset, and then averages the resulting $m$ SVMs in order to obtain the final predictor.

In the localized approach, one also trains SVMs on subsets of \Dn, but the split of \Dn is now obtained in a spatial way---based on some regionalization of the input space \X---instead of randomly. Following early theoretical investigations of such localized approaches \citep{bottou1992,vapnik1993}, different methods for obtaining the required regions have been examined. These include decision trees \citep[\eg,][]{bennett1998,wu1999,tibshirani2007,chang2010}, $k$-nearest neighbors ($k$NN) methods \citep[\eg,][]{zhang2006,blanzieri2007,blanzieri2008,segata2010,hable2013} as well as variants of $k$-means \citep[\eg,][]{cheng2010,gu2013}. In comparison to distributed learning, this has the disadvantage that, no matter which method of regionalization is chosen, the process of regionalizing the input space clearly also takes some time for large data sets---albeit considerably less time than just training an SVM on the whole data set---, thus making the computational gain of such a localized approach in the training phase smaller than that of distributed learning. On the other hand, the evaluation of the resulting predictor for a test sample can actually be \emph{significantly faster} in localized approaches than it is in distributed ones: Whereas one has to evaluate each of the $m$ different SVMs (and then average the results) in distributed learning, it suffices to evaluate the one SVM belonging to the region of the test sample in localized learning (if the regions do not overlap). 

Furthermore, localizing the SVM approach can also yield advantages regarding the \emph{quality of prediction}---compared to distributed learning as well as regular SVMs: Whereas the underlying true function, which one aims to estimate, can of course exhibit discontinuities, SVMs based on a continuous and bounded kernel such as the commonly used Gaussian RBF kernel are always continuous (and bounded) themselves, \vgl \citet[Lemma~4.28]{steinwart2008}. This can lead to SVMs not accurately modeling the true function near such discontinuities, but instead greatly oscillating and overshooting---an effect that is also known from Fourier series, where it is called the Gibbs phenomenon, \vgl \citet{hewitt1979}. Additionally, in global learning approaches like SVMs, the complexity of the predictor is usually controlled globally by a very small amount of hyperparameters. Hence, an accurate prediction can be difficult for such global approaches if the complexity and variability of the true function, or that of the conditional distributions $\Pbed[Y]{X=x}$, greatly differ between different areas of the input space \X, even if the true function does not exhibit any discontinuities. Both of these problems can be overcome by the use of localized methods, as a good regionalization can split the input space into separate regions at (or at least close to) discontinuities and such that the complexity and variability do not change too much throughout the individual regions, see also \Cref{Abb:LocCons_Loc_AdvantageLocalization}.

\begin{figure}[t]
	\begin{center}
		\vspace*{-0.5cm}\includegraphics[width=\textwidth]{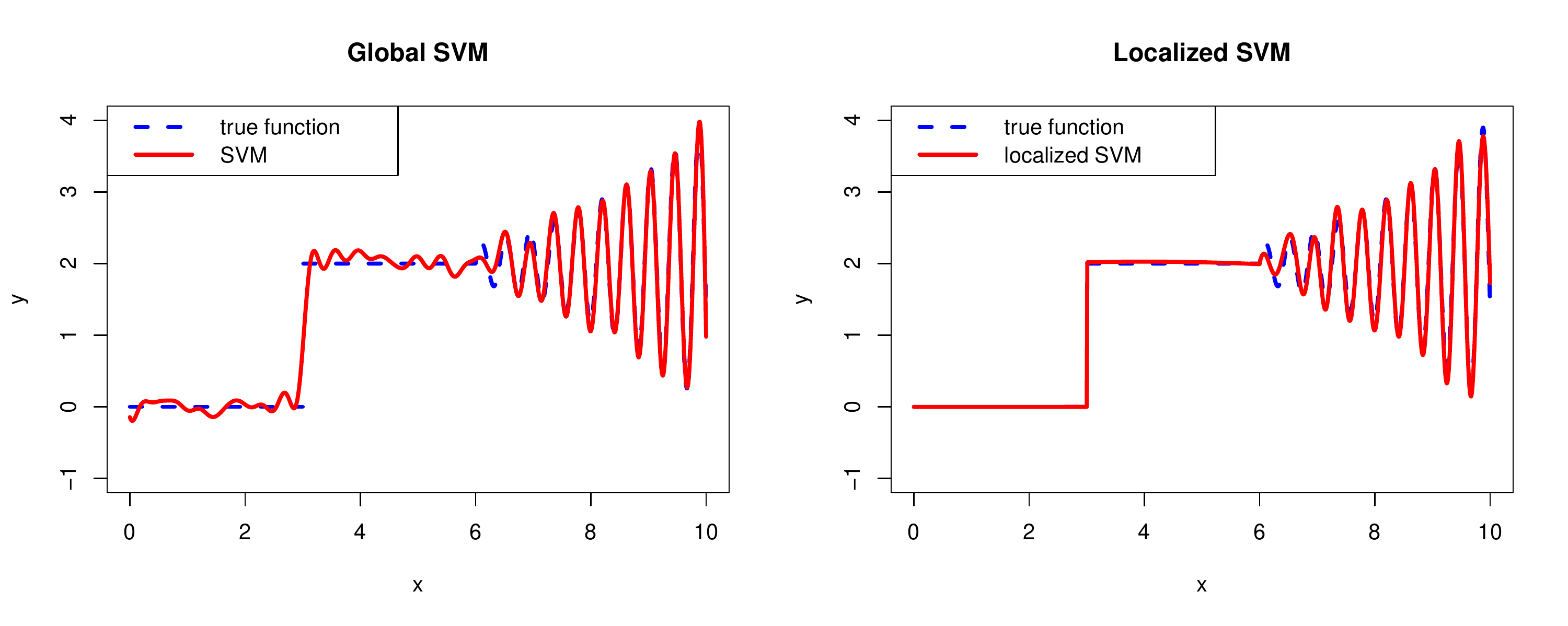}
		\vspace*{-0.5cm}\caption{A global SVM (left plot) and a localized SVM (right plot; splits between the regions at $x=3$ and $x=6$) fitted to the same data which was generated according to the plotted true function and some normally distributed error. The global SVM (slightly) overshoots at the discontinuity at $x=3$ and oscillates too much for $x\le 6$ because the underlying hyperparameters have to be chosen in a way that also allows for a reasonably good fit for $x>6$, where the true function oscillates very quickly. The localized SVM does not exhibit these problems and yields a considerably better fit overall.}
		\label{Abb:LocCons_Loc_AdvantageLocalization}
	\end{center}
\end{figure}

This intuition of localized SVMs also being able to improve regular SVMs with regard to the quality of prediction gets affirmed by \citet{blaschzyk2022}, who, in the case of using the hinge loss for classification, derived learning rates exceeding those known for regular SVMs. Whereas most of the papers on localized SVMs mentioned in the preceding paragraphs focus on the experimental analysis of a specific method of localization, \citet{blaschzyk2022} constitutes an example of a paper deriving theoretical results and additionally not requiring any special method of localization (instead only requiring the resulting regionalization to satisfy some conditions which are often quite mild). There are several papers taking a similar approach and also deriving learning rates for such localized SVMs, with \citet{thomann2017} also using the hinge loss and \citet{meister2016,muecke2019} investigating least squares regression. 

Whereas learning rates of course also imply (risk) consistency, they always require additional assumptions regarding the unknown probability measure \P because of the no-free-lunch theorem \citep[\vgl][]{devroye1982}, and most of the mentioned papers for example additionally require \X to be contained in some ball and \Y to be bounded as well. We however take an approach similar to \citet{dumpert2018,dumpert2020,koehler2022} who allowed for even more general regionalizations as well as more general kernels and loss functions and did not impose any restrictive assumptions regarding \P, and who then proved that localized SVMs are risk consistent (which we in some aspects considerably generalize in \Cref{Sec:LocCons_Cons}), statistically robust with respect to the maxbias as well as the influence function, and totally stable with respect to simultaneous changes in not only the probability measure but also the regularization parameter, the kernel and the regionalization. We derive results on \Lp- as well as risk consistency in \Cref{Sec:LocCons_Cons}.

\subsection{Prerequisites regarding localized SVMs}\label{SubSec:LocCons_Loc_Pre}

Before stating our results in \Cref{Sec:LocCons_Cons}, we first have to formally define localized SVMs as well as to specify the mild assumptions which we need to impose upon the regionalizations in order to be able to then derive our results.

As already mentioned, we actually allow for regionalizations that change with $n$. For $n\in\N$, we define the regionalization \Xnbold as $\Xnbold:=\{\Xni[1],\dots,\Xni[m_n]\}$ for sets $\Xni[1],\dots,\Xni[m_n]$. We further denote $\Xnbold(x):=\{\tilde{\X}\in\Xnbold\,|\,x\in\tilde{\X}\}$ for all $x\in\X$ and $n\in\N$, and assume the following three conditions to hold true:
\begin{description}[labelwidth=\widthof{\bfseries(R3)}]
	\item[(R1)] $\Xni[1],\dots,\Xni[m_n] \subseteq \X$ complete (as metric spaces) and measurable such that $\X=\bigcup_{i=1}^{m_n} \Xni$ for all $n\in\N$.
	\item[(R2)] $\exists\,\smax\in\N$ such that $|\Xnbold(x)|\le \smax$ for all $x\in\X$ and $n\in\N$.
	\item[(R3)] The sequence $(\Xnbold)_{n\in\N}$ is stochastically independent of the sequence $(\Dn)_{n\in\N}$ of training data sets.
\end{description}

\begin{mybem}
	Condition \textbf{(R3)} might seem restrictive at first glance because it seemingly constitutes a restriction to only using regionalizations whose construction does not take the observed data into account. However, one can easily circumvent this restriction by randomly partitioning the whole data set into not only the usual three parts---namely a training data set \Dn, a validation data set and a test data set---but four parts instead, where the fourth part is a regionalization data set. This way, the regionalizations can be chosen data-dependently without violating \textbf{(R3)}. By putting only a relatively small part of the available data into the regionalization data set---because one reason for regionalizing is to just reduce the subsequent training time of the SVMs, for which no ``perfect'' regionalization is necessary---, this procedure does not substantially reduce the amount of data available for training, validating and testing.
\end{mybem}

Note that \textbf{(R1)} tells us that, for every $n\in\N$, the regions need not necessarily be pairwise disjoint but can instead also overlap---as long as \textbf{(R2)} is satisfied, that is, as long as the number of regions overlapping does not exceed some global constant \smax in any point $x\in\X$. If the regionalization does not change with $n$, then \textbf{(R2)} is trivially satisfied for $\smax=m_1$.

\begin{mybem}
	By \citet[Lemma I.6.4 and Theorem I.6.12]{dunford1957}, any subset of a separable metric space is a separable metric space again if it is equipped with the metric of the original space. Hence, \Cref{Ann:LocCons_Pre_AllgAnn} being satisfied for \X implies it also being satisfied for the regions $\Xni$, $n\in\N$ and $i\in\{1,\dots,m_n\}$.
\end{mybem}

In order to define local SVMs on the different regions, we need to have a probability measure on each of these regions. It suggests itself to define these measures by restricting \P. For $n\in\N$ and $i\in\{1,\dots,m_n\}$, we define the local measure \Pni on $\Xni\times\Y$ by
\begin{align*}
	\Pni := \begin{cases}
		\frac{1}{\P(\Xni\times\Y)}\cdot \P\einschraenkung_{\Xni\times\Y} &\wenn \P(\Xni\times\Y)>0\\
		0 &\sonst.
	\end{cases}
\end{align*}
This obviously only is a probability measure if $\P(\Xni\times\Y)>0$, but we will see that we can mostly ignore the regions with $\P(\Xni\times\Y)=0$ for our results. We denote
\begin{align*}
	\IndexPpositive:=\big\{i\in\{1,\dots,m_n\}\,|\,\P(\Xni\times\Y)>0\big\}
\end{align*} 
and $\AnzPpositive:=|\IndexPpositive|$ for $n\in\N$. Similarly, we define the local empirical measures $\DVertni$ by
\begin{align*}
	\DVertni := \begin{cases}
		\frac{1}{\DVertn(\Xni\times\Y)}\cdot \DVertn\einschraenkung_{\Xni\times\Y} &\wenn \DVertn(\Xni\times\Y)>0\\
		0 &\sonst,
	\end{cases}
\end{align*}
such that (if $\DVertn(\Xni\times\Y)>0$) they are the empirical probability measures associated with the subsets $\Dni:=\Dn\cap(\Xni\times\Y)$ of \Dn, for which we denote $\dni:=|\Dni|$.

As mentioned before, one of the goals behind this localized approach is to increase the method's capability to accurately learn a function whose complexity and variability differ between different areas of the input space, by separating these areas into different regions. Since a principal mechanism for controlling the complexity of an SVM is the choice of the regularization parameter and of the kernel (respectively the hyperparameters of the kernel), one should therefore also be allowed to choose different regularization parameters and kernels in the different regions. We hence have, for each $n\in\N$, a vector of regularization parameters $\lbnbold := (\lbni[1],\dots,\lbni[m_n])$, with $\lbni>0$ for all $i\in\{1,\dots,m_n\}$, and a vector of kernels $\knbold := (\kni[1],\dots,\kni[m_n])$, where \kni is a kernel on \Xni for each $i\in\{1,\dots,m_n\}$. 

Based on the regularization parameters, kernels and a loss function \loss, one obtains from \eqref{eq:LocCons_Pre_theoSVM} SVMs 
\begin{align*}
	\fnitheoext \colon \Xni\to\R\,, \qquad n\in\N,\, i\in\{1,\dots,m_n\}\,, 
\end{align*}
which we call \textit{local SVMs} on \Xni. If \Pni is the zero measure, the above SVM is undefined and we just define it as the zero function, $\fnitheoext\equiv 0$, in this case. Analogously, we define the local empirical SVMs 
\begin{align*}
	\fniempext \colon \Xni\to\R\,, \qquad n\in\N,\, i\in\{1,\dots,m_n\}\,,
\end{align*}
as in \eqref{eq:LocCons_Pre_empSVM}, with $\fniempext\equiv 0$ if \DVertni is the zero measure.

Since we want to combine these local SVMs in order to obtain a global predictor on \X, we first need to extend them in a way such that they are defined on all of \X. That is, for all functions $g\colon\tilde{\X}\to\R$ on $\tilde{\X}\subseteq\X$, we define the zero-extension $\hat{g}\colon\X\to\R$ by
\begin{align*}
	\hat{g}(x) := \begin{cases}
		g(x) &\wenn x\in\tilde{\X}\,,\\
		0 &\sonst\,.
	\end{cases}
\end{align*}

Now, all that is left to do in order to obtain our global predictors, is to equip the local SVMs with weight functions which pointwisely control the influence of each local SVM in areas where two or more regions overlap. We only impose the following three standard assumptions for weight functions on them:
\begin{description}
	\item[(W1)] $\wni\colon\X\to[0,1]$ measurable for all $i\in\{1,\dots,m_n\}$ and $n\in\N$.
	\item[(W2)] $\sum_{i=1}^{m_n} \wni(x) = 1$ for all $x\in\X$ and $n\in\N$.
	\item[(W3)] $\wni(x)=0$ for all $x\notin\Xni$ and all $i\in\{1,\dots,m_n\}$ and $n\in\N$.
\end{description}

Our global predictor \fntheoext, which we call \textit{localized SVM} even though it is not necessarily an SVM itself, is then defined by
\begin{align}\label{eq:LocCons_Loc_DefGlobalPredictor}
	\fntheoext\colon \X\to\R\,,\, x\mapsto \sum_{i=1}^{m_n} \wni(x) \cdot \fnitheoextdach(x)
\end{align}
for $n\in\N$. Analogously, we define the \textit{empirical localized SVM} 
\begin{align}\label{eq:LocCons_Loc_DefGlobalPredictorEmp}
	\fnempext\colon \X\to\R\,,\, x\mapsto \sum_{i=1}^{m_n} \wni(x) \cdot \fniempextdach(x)
\end{align}
for $n\in\N$.

Finally, before stating the consistency results for localized SVMs in \Cref{Sec:LocCons_Cons}, we introduce the concept of \textit{families of kernels of type \betabold} which will be needed in those results.

\begin{mydef}\label{Def:LocCons_Loc_KernelFamily}
	Let $I$ be an index set such that $0\in I$. For kernels $k^{(r)}$ and constants $\beta^{(r)}\in(0,\infty)$, $r\in I$, we say that $\kbold:=(\k^{(r)})_{r\in I}$ is a \textit{family of kernels of type $\betabold:=(\beta^{(r)})_{r\in I}$} if, for all $r\in I$, 
	\begin{enumerate}[label=(\roman*)]
		\item $\H^{(r)}\supseteq\H^{(0)}$, where $\H^{(r)}$ and $\H^{(0)}$ are the RKHSs associated with $\k^{(r)}$ and $\k^{(0)}$ respectively, and
		\item $\norm{\H^{(r)}}{f}\le\beta^{(r)} \cdot \norm{\H^{(0)}}{f}$ for all $f\in\H^{(0)}$.
	\end{enumerate}
\end{mydef}

\begin{mybem}\label{Bem:LocCons_Loc_KernelFamily}
	By \citet[Theorem 2.17]{saitoh2016} (see also Part I.7 of \citealp{aronszajn1950}, and Section 4.5 of \citealp{berlinet2004}, for related considerations), condition (i) from \Cref{Def:LocCons_Loc_KernelFamily} already implies that there exists some $\beta^{(r)}\in(0,\infty)$ such that (ii) is satisfied as well. Hence, every family of kernels satisfying (i) will also be a family of kernels of type \betabold for suitable \betabold. Furthermore, the same theorem also yields that the two conditions from \Cref{Def:LocCons_Loc_KernelFamily} are equivalent to
	\begin{enumerate}
		\item[(iii)] $(\beta^{(r)})^2\cdot\k^{(r)} - \k^{(0)}$ is a kernel,
	\end{enumerate}
	for which reason families of kernels of type \betabold are equivalently characterized by (iii) holding true for all $r\in I$.
\end{mybem}

\begin{mybsp}\label{Bsp:LocCons_Loc_GRBF}
	Let $d\in\N$, $\X\subseteq\R^d$ non-empty and $I$ be an index set such that $0\in I$. For $r\in I$, define $\k^{(r)}$ as the Gaussian kernel with bandwidth $\gamma^{(r)}\in(0,\infty)$, that is,
	\begin{align*}
		k^{(r)}(x,x') := \exp\left(-\frac{\norm{2}{x-x'}^2}{(\gamma^{(r)})^2}\right) \qquad \forall\, x,x'\in\X\,. 
	\end{align*}
	By \citet[Proposition 4.46]{steinwart2008}, the conditions from \Cref{Def:LocCons_Loc_KernelFamily} are satisfied with $\beta^{(r)}:=(\gamma^{(0)}/\gamma^{(r)})^{d/2}$ if $\gamma^{(0)}\ge\sup_{r\in I\setminus\{0\}}\gamma^{(r)}$. 
	
	Hence, every family $(\k^{(r)})_{r\in J}$, $0\notin J$, of Gaussian kernels with bounded bandwidth can be turned into a family of kernels of type $\betabold=((\gamma^{(0)}/\gamma^{(r)})^{d/2})_{r\in I}$, $I:=J\cup\{0\}$, by choosing $\k^{(0)}$ as the Gaussian kernel with bandwidth $\gamma^{(0)}=\sup_{r\in J}\gamma^{(r)}$.
\end{mybsp}

We introduced these families of kernels of type \betabold since we will require all kernels \kni, $n\in\N$, $i\in\{1,\dots,m_n\}$, used in the local SVMs to come from the union of $\ell\in\N$ such families $\kjbold[1],\dots,\kjbold[\ell]$. To be more specific, \kjbold, $j=1,\dots,\ell$, will consist of kernels on \X and each \kni will be the restriction of such a kernel to $\Xni\times\Xni$. That is, we will have $\kni=\k^{(j_0,r_0)}\einschraenkung_{\Xni\times\Xni}$ for some $j_0\in\{1,\dots,\ell\}$ and $r_0\in\Ij[j_0]$, where $\Ij[j_0]$ denotes the index set of the $j_0$-th family. Based on this, we introduce the additional notation $\betani:=\beta^{(j_0,r_0)}$ and $\kninull:=\k^{(j_0,0)}\einschraenkung_{\Xni\times\Xni}$ (in case of ambiguity regarding $j_0$ and $r_0$, any of the options may be chosen), which will be needed later on. 

Note that the concept of families of kernels of type \betabold also allows for infinite index sets (see also \Cref{Bsp:LocCons_Loc_GRBF}). This will lead to the kernels \kni, $n\in\N$, $i\in\{1,\dots,m_n\}$, being allowed to be chosen from an possibly infinite set of kernels.

\section{Consistency of Localized SVMs}\label{Sec:LocCons_Cons}

In the following, we first derive \Lp-consistency and afterwards risk consistency of localized SVMs as defined in \Cref{SubSec:LocCons_Loc_Pre}. To our knowledge, there do not exist any results on \Lp-consistency of localized SVMs so far, and whereas there do exist results on their risk consistency, our result significantly generalizes those in several ways. Before stating the results, we impose the following assumptions, which we assume to hold true throughout this section:

\begin{myann}\label{Ann:LocCons_Cons_Ann}\phantom{text}
	\begin{itemize}
		\item Let $\loss\colon\YR\to[0,\infty)$ be a convex, distance-based loss function of growth type $p\in[1,\infty)$.
		\item Let $\Xnbold:=\{\Xni[1],\dots,\Xni[m_n]\}$, $n\in\N$, be regionalizations satisfying \textbf{(R1)}, \textbf{(R2)}, \textbf{(R3)}, and let \wni, $n\in\N$ and $i=1,\dots,m_n$, be weight functions satisfying \textbf{(W1)}, \textbf{(W2)}, \textbf{(W3)}.
		\item Let $\ell\in\N$ and let, for $j=1,\dots,\ell$, $\kjbold:=(\kjr)_{r\in\Ij}$ be a family of uniformly bounded and measurable kernels of type $\betajbold:=(\betajr)_{r\in\Ij}$ on \X with separable RKHSs $(\Hjr)_{r\in\Ij}$ such that $\Hjr[0]\subseteq\Lppxvar$ dense. Let, for all $n\in\N$ and $i\in\{1,\dots,m_n\}$,
		\begin{align*}
			\kni\in\left\{\kjr\einschraenkung_{\Xni\times\Xni} \,:\, j\in\{1,\dots,\ell\}, r\in\Ij\right\}\,.
		\end{align*} 
		\item Assume $|\P|_p<\infty$ and $\sup_{n\in\N,i\in\IndexPpositive} |\Pni|_p <\infty$.
	\end{itemize}
\end{myann}

\begin{mybem}\label{Bem:LocCons_Cons_MomentenAlternativen}
	The condition $\sup_{n\in\N,i\in\IndexPpositive} |\Pni|_p <\infty$ is disadvantageous in that it requires knowledge about all regionalizations \Xnbold, $n\in\N$. Because 
	\begin{align*}
		|\Pni|_p^p = \int_{\Xni} |\Pbed{x}|_p^p \diff\Pnix(x) \le \sup_{x\in\X}|\Pbed{x}|_p^p
	\end{align*}
	for all $n\in\N$ and $i\in\IndexPpositive$ (and analogously also $|\P|_p^p\le\sup_{x\in\X}|\Pbed{x}|_p^p$), it however suffices if $\sup_{x\in\X}|\Pbed{x}|_p<\infty$. 
	
	On the other hand, even though the finiteness of $|\P|_p$ does already imply the finiteness of $|\Pni|_p$ for all $n\in\N$ and $i\in\IndexPpositive$ because 
	\begin{align*}
		|\Pni|_p^p &= \int_{\Xni} |\Pbed{x}|_p^p \diff\Pnix(x) = \frac{1}{\Pxvar(\Xni)} \cdot \int_{\Xni} |\Pbed{x}|_p^p \diff\Pxvar(x)\\ 
		&\le \frac{1}{\Pxvar(\Xni)} \cdot \int_{\X} |\Pbed{x}|_p^p \diff\Pxvar(x) = \frac{1}{\Pxvar(\Xni)} \cdot |\P|_p^p\,,
	\end{align*}
	$|\P|_p$ being finite is not sufficient to guarantee $\sup_{n\in\N,i\in\IndexPpositive} |\Pni|_p <\infty$, as can be seen from the following example:
	
	Let $\Pxvar:=\unif$ and $\Pbed{X=x}:=\unif[0,x^{-1/2}]$ for all $x\in(0,1)$, where $\unif[a,b]$ denotes the uniform distribution on $(a,b)$. Then, we have
	\begin{align*}
		|\P|_1 = \int_{0}^{1}\int_{0}^{\frac{1}{\sqrt{x}}} y\sqrt{x} \diff y\diff x = 1 < \infty\,,
	\end{align*}
	but for $\Xni[1]:=(0,\frac{1}{n})$, $n\in\N$, we obtain
	\begin{align*}
		|\Pni[1]|_1 = \int_{0}^{\frac{1}{n}}\int_{0}^{\frac{1}{\sqrt{x}}} y\sqrt{x} \diff y \cdot n \diff x = \sqrt{n}\,,
	\end{align*}
	which yields $\sup_{n\in\N,i\in\IndexPpositive} |\Pni|_p =\infty$.
	
	Hence, the condition $\sup_{n\in\N,i\in\IndexPpositive} |\Pni|_p <\infty$ is not superfluous in itself and can not just be erased without adding a replacement like $\sup_{x\in\X}|\Pbed{x}|_p<\infty$.
\end{mybem}

The subsequent theorem shows that localized SVMs are indeed \Lp-consistent under \Cref{Ann:LocCons_Cons_Ann}.

\begin{mythm}\label{Thm:LocCons_Cons_LpCons}
	Let \Cref{Ann:LocCons_Pre_AllgAnn,Ann:LocCons_Cons_Ann} be satisfied. Let \fnempext, $n\in\N$, be defined as in \eqref{eq:LocCons_Loc_DefGlobalPredictorEmp} and assume that \fbayes is \Pxvar-\fs unique.
	Define $p_1^*:=\max\{p+1,p(p+1)/2\}$. Further choose $p_2^*:=\max\{2(p-1)/p,p-1\}$ if $p>1$ and $p_2^*\in (0,\infty)$ arbitrary if $p=1$. If the regularization parameters satisfy $\lbni\in(0,C)$ for all $n\in\N$ and $i\in\{1,\dots,m_n\}$ for some $C\in(0,\infty)$, as well as
	\begin{align}\label{eq:Thm_LocCons_Cons_LpCons_ErsteBedingungLambda}
		\max_{i\in\IndexPpositive}\betani^2\lbni\to 0
	\end{align}
	and 
	\begin{align}\label{eq:Thm_LocCons_Cons_LpCons_ZweiteBedingungLambda}
		\min_{i\in\IndexPpositive} \frac{\lbni^{p_1^*}\dni}{\AnzPpositive^{p_2^*}} \to \infty
	\end{align}
	as $n\to\infty$, then
	\begin{equation*}
		\limn \normLppxvar{\fnempext-\fbayes} = 0 \qquad \text{in probability $\P^\infty$.}
	\end{equation*}
\end{mythm}

\begin{mybsp}\label{Bsp:LocCons_Cons_p1p2}
	If $p=2$, like for the popular least squares loss, we have $p_1^*=3$ and $p_2^*=1$ and condition \eqref{eq:Thm_LocCons_Cons_LpCons_ZweiteBedingungLambda} therefore becomes
	\begin{align*}
		\min_{i\in\IndexPpositive} \frac{\lbni^3\dni}{\AnzPpositive} \to \infty\,.
	\end{align*}
	If $p=1$, like for the pinball loss or the \eps-insensitive loss, we have $p_1^*=2$ and $p_2^*$ can be chosen arbitrarily small. Hence, condition \eqref{eq:Thm_LocCons_Cons_LpCons_ZweiteBedingungLambda} relaxes even further in this case, becoming
	\begin{align*}
		\min_{i\in\IndexPpositive} \frac{\lbni^2\dni}{\AnzPpositive^{\eps}} \to \infty\,
	\end{align*}
	for an arbitrarily small $\eps>0$.
\end{mybsp}

\begin{mybem}\label{Bem:LocCons_Cons_SpecialCases}
	In some special cases, we can slightly simplify the conditions regarding the regularization parameters in \Cref{Thm:LocCons_Cons_LpCons}:
	
	If one only allows for a finite amount of kernels to choose from (instead of a finite amount of families of kernels of type \betabold), it is obviously possible to view each of these kernels as its own family of kernels with index set $\Ij=\{0\}$ and $\betajr[0]=1$ for all $j\in\{1,\dots,\ell\}$, and thus simplify \eqref{eq:Thm_LocCons_Cons_LpCons_ErsteBedingungLambda} by eliminating \betani from it. 
	
	Additionally, if the regionalization \Xnbold does not change with $n$, then $\AnzPpositive$ is constant and we can erase it from \eqref{eq:Thm_LocCons_Cons_LpCons_ZweiteBedingungLambda}.
	
	Hence, if both of these hold true (finite amount of kernels and constant regionalization), the conditions regarding the regularization parameters are exactly the same as in \citet{koehler2023arxiv}, where \Lp-consistency of non-localized SVMs was derived, with the only difference being that the conditions obviously need to hold true for each region now instead of only globally.
\end{mybem}

Now, we can turn our attention to risk consistency of localized SVMs. To our knowledge, the only existing results which explicitly examine risk consistency of localized SVMs are Theorem~1 from \cite{hable2013} and Theorem~3.1 from \citet{dumpert2018}, both of which are in certain aspects considerably less general than the subsequent \Cref{Thm:LocCons_Cons_LpCons}: \citet{dumpert2018} only considered Lipschitz continuous (shifted) loss functions, whereas we take a look at distance-based loss functions, thus covering a different subset of all loss functions, notably also including the popular and not Lipschitz continuous least squares loss. Additionally, \citet{dumpert2018} assumed a fixed regionalization and fixed kernels on the different regions, which stay the same independently of the size $n$ of the underlying data set. We however also allow for regionalizations which change with $n$ (\vgl \Cref{SubSec:LocCons_Loc_Pre}), since the regionalization is oftentimes not predefined in practice but instead might change when new data points are added to the data set---for example, becoming finer when $n$ grows. We also allow for kernels that change with $n$ and that are chosen from an possibly infinite set of kernels---for example, Gaussian kernels whose bandwidth decreases as $n$ increases (\vgl \Cref{Bsp:LocCons_Loc_GRBF}). Thus, we significantly generalize the investigations from \citet{dumpert2018} in these aspects. \citet{hable2013} on the other hand only allows for a bounded output space \Y and only considers the special case of the regionalization stemming from some $k$-nearest neighbor method. Whereas this approach implicitly also allows for regionalizations which change with $n$, this makes our \Cref{Thm:LocCons_Cons_RiskCons} applicable to a much wider array of localization methods---even though the $k$-nearest neighbor approach described by \citet{hable2013} is not one of them because it can lead to condition \textbf{(R2)} from \Cref{SubSec:LocCons_Loc_Pre} being violated, thus making our result and that of \citet{hable2013} applicable to different situations. 

Apart from that, the oracle inequalities from \citet{meister2016,thomann2017,muecke2019,blaschzyk2022} of course also imply risk consistency if the different parameters in these results are chosen accurately. However, these oracle inequalities are only valid for the least squares respectively the hinge loss, whereas we aim at deriving a much more general result which is applicable for the considerably larger class of convex, distance-based loss functions. Additionally, these oracle inequalities require stricter conditions than our consistency results, like for example \X being contained in a ball of fixed radius, \Y being bounded, the kernels all being Gaussian kernels, and also additional requirements regarding the regionalization. 

In the subsequent theorem, we derive such a general result on the risk consistency of localized SVMs. Condition \eqref{eq:Thm_LocCons_Cons_RiskCons_BedingungLambda} in that theorem is slightly more restrictive and complicated than its counterpart \eqref{eq:Thm_LocCons_Cons_LpCons_ZweiteBedingungLambda} in the result on \Lp-consistency. However, the additional factor $\lbni[j]^{p_3^*}$ can be eliminated from \eqref{eq:Thm_LocCons_Cons_RiskCons_BedingungLambda} in several important special cases, thus weakening and simplifying this condition again: If the loss function is of growth type $p=1$, one directly obtains $p_3^*=0$, and if the regionalizations underlying the localized SVMs partition \X or \fbayes is \Pxvar-\fs unique, the special cases (i) and (ii) of the theorem also yield similar relaxations.

\begin{mythm}\label{Thm:LocCons_Cons_RiskCons}
	Let \Cref{Ann:LocCons_Pre_AllgAnn,Ann:LocCons_Cons_Ann} be satisfied. Let \fnempext, $n\in\N$, be defined as in \eqref{eq:LocCons_Loc_DefGlobalPredictorEmp}. Define $p_1^*:=\max\{p+1,p(p+1)/2\}$ and $p_3^*:=\max\{p-1,p(p-1)/2\}$. Further choose $p_2^*:=\max\{2(p-1)/p,p-1\}$ if $p>1$ and $p_2^*\in (0,\infty)$ arbitrary if $p=1$. If the regularization parameters satisfy $\lbni\in(0,C)$ for all $n\in\N$ and $i\in\{1,\dots,m_n\}$ for some $C\in(0,\infty)$, as well as $\max_{i\in\IndexPpositive}\betani^2\lbni\to 0$ and
	\begin{align}\label{eq:Thm_LocCons_Cons_RiskCons_BedingungLambda}
		\min_{i,j\in\IndexPpositive} \frac{\lbni[j]^{p_3^*}\lbni^{p_1^*}\dni}{\AnzPpositive^{p_2^*}} \to \infty
	\end{align}
	as $n\to\infty$, then
	\begin{equation*}
		\limn \risk(\fnempext) = \riskbayes \qquad \text{in probability $\P^\infty$.}
	\end{equation*}
	If some additional conditions are satisfied, it is possible to slightly relax assumption \eqref{eq:Thm_LocCons_Cons_RiskCons_BedingungLambda} regarding the regularization parameters:
	\begin{enumerate}[label=(\roman*)]
		\item If, for all $n\in\N$, the regionalization \Xnbold is a partition of \X, then it suffices if \eqref{eq:Thm_LocCons_Cons_RiskCons_BedingungLambda} is satisfied for $p_1^*:=\max\{2p,p^2\}$ and $p_3^*:=0$.
		\item If \fbayes is \Pxvar-\fs unique, then it suffices if \eqref{eq:Thm_LocCons_Cons_RiskCons_BedingungLambda} is satisfied for $p_3^*:=0$.
	\end{enumerate}
\end{mythm}

If $p=1$, the cases (i) and (ii) can be ignored since they do not yield an actual relaxation because $p_3^*=0$ then also holds true in the general case. Furthermore, the possible relaxations mentioned in \Cref{Bem:LocCons_Cons_SpecialCases} are obviously also valid for \Cref{Thm:LocCons_Cons_RiskCons}.


\section{Discussion}\label{Sec:LocCons_Discussion}

In this paper, the \Lp- and risk consistency of localized SVMs has been investigated, as localized SVMs can offer reduced computational requirements as well as advantages regarding the quality of the predictions over non-localized SVMs (\vgl \Cref{SubSec:LocCons_Loc_Overview}). We saw that it is possible to derive both types of consistency of localized SVMs under very mild conditions on the underlying probability distribution as well as the applied regionalization and the kernels used in the different local SVMs. Notably, we even allowed for regionalizations which change as the size $n$ of the data set increases---in contrast to \cite{dumpert2018}, where risk consistency of localized SVMs had already been examined, but only for non-changing regionalizations and kernels and for a different subset of loss functions. Hence, we added another entry to the list of properties that localized SVMs inherit from non-localized ones. This further justifies applying localized SVMs to learning problems, especially to those in which non-localized methods struggle, like in big data scenarios or if the function which one wishes to estimate contains discontinuities or exhibits greatly differing complexity and variability across different areas of the input space.

\section*{Acknowledgments}

I would like to thank my PhD supervisor Andreas Christmann for helpful discussions on this topic. The work described in this paper was partially supported by grant CH291/3-1 of the Deutsche Forschungsgesellschaft.

\appendix

\section{Auxiliary Results}\label{Sec:LocCons_Aux}

In this section, we prove auxiliary results that are needed in the proofs of \Cref{Thm:LocCons_Cons_LpCons} and \Cref{Thm:LocCons_Cons_RiskCons}. In both these results, the difference between \fnempext and \fbayes is examined---the \Lp-norm of the difference in the former and the difference between the risks in the latter. In both cases, we do not examine this difference directly, but instead plug in the theoretical localized SVM \fntheoext as an intermediate step and then examine the difference between \fnempext and \fntheoext as well as that between \fntheoext and \fbayes. The lemmas from this section deal with these differences.

As the assumptions needed for these lemmas are slightly weaker than those needed in the theorems from \Cref{Sec:LocCons_Cons} (and additionally differ between these lemmas), \Cref{Ann:LocCons_Cons_Ann} is \textit{not} assumed to hold true in this section, but we will instead explicitly list the required assumptions in the lemmas.

\begin{mylem}\label{Lem:LocCons_Aux_LocSVMsSupNormEmpTheo}
	Let \Cref{Ann:LocCons_Pre_AllgAnn} be satisfied. Let $\loss\colon\YR\to[0,\infty)$ be a convex, distance-based loss function of upper growth type $p\in[1,\infty)$. Let \fntheoext and \fnempext, $n\in\N$, be defined as in \eqref{eq:LocCons_Loc_DefGlobalPredictor} and \eqref{eq:LocCons_Loc_DefGlobalPredictorEmp} such that the underlying regionalizations and weight functions satisfy \textbf{(R1)}, \textbf{(R3)}, \textbf{(W1)}, \textbf{(W2)}, \textbf{(W3)} and $\sup_{n\in\N,i\in\IndexPpositive} |\Pni|_p <\infty$. Assume that, for all $n\in\N$ and $i\in\{1,\dots,m_n\}$, \kni is a bounded and measurable kernel on \Xni with separable RKHS \Hni, such that $\sup_{n\in\N,i\in\IndexPpositive} \normSup{\kni} < \infty$. Define $p_1^*:=\max\{p+1,p(p+1)/2\}$. Further choose $p_2^*:=\max\{2(p-1)/p,p-1\}$ if $p>1$ and $p_2^*\in (0,\infty)$ arbitrary if $p=1$. If the regularization parameters satisfy $\lbni\in(0,C)$ for all $n\in\N$ and $i\in\{1,\dots,m_n\}$ for some $C\in(0,\infty)$, as well as 
	\begin{align}\label{eq:Lem_LocCons_Aux_LocSVMsSupNormEmpTheo_BedingungLambda}
		\min_{i\in\IndexPpositive} \frac{\lbni^{p_1^*}\dni}{\AnzPpositive^{p_2^*}} \to \infty
	\end{align}
	as $n\to\infty$, then
	\begin{align*}
		\limn \normLinftypxvar{\fnempext-\fntheoext} = 0 \qquad \text{in probability $\P^\infty$.} 
	\end{align*}
\end{mylem}

\begin{proof}
	To shorten the notation, we will denote $\fnitheo:=\fnitheoext$ and $\fniemp:=\fniempext$ for all $n\in\N$ and $i\in\{1,\dots,m_n\}$, as well as $\kappa:=\sup_{n\in\N,i\in\IndexPpositive} \normSup{\kni}$ and $\rho:=\sup_{n\in\N,i\in\IndexPpositive} |\Pni|_p$ throughout this proof.
	
	Because applying \textbf{(W1)} and \textbf{(W2)} yields
	\begin{align*}
		&\left| \fnempext(x) - \fntheoext(x) \right| = \left| \sum_{i=1}^{m_n} \wni(x) \cdot \left( \fniempdach(x) - \fnitheodach(x) \right) \right|\\ 
		&\le \sum_{i=1}^{m_n} \wni(x) \cdot \left| \fniempdach(x) - \fnitheodach(x) \right| \le \max_{i\in\{1,\dots,m_n\}} \left| \fniempdach(x) - \fnitheodach(x) \right|
	\end{align*}
	for all $n\in\N$ and all $x\in\X$, we obtain
	\begin{align}\label{eq:ProofLem_LocCons_Aux_LocSVMsSupNormEmpTheo_Aufteilung}
		&\normLinftypxvar{\fnempext-\fntheoext} \le \max_{i\in\{1,\dots,m_n\}} \normLinftypxvar{\fniempdach-\fnitheodach}\notag\\
		&= \max_{i\in\IndexPpositive} \normLinftypnix{\fniemp-\fnitheo} \le \kappa \cdot \max_{i\in\IndexPpositive} \norm{\Hni}{\fniemp-\fnitheo}
	\end{align}
	for all $n\in\N$, with the last inequality holding true because of \citet[Lemma~4.23]{steinwart2008}. Hence, we start by fixing an $n\in\N$ and an $i\in\IndexPpositive$ and investigating the corresponding difference on the right hand side of \eqref{eq:ProofLem_LocCons_Aux_LocSVMsSupNormEmpTheo_Aufteilung}.
	
	First, note that employing \citet[Lemma~4.23, equation~(5.4) and Lemma~2.38(i)]{steinwart2008} yields
	\begin{align}\label{eq:ProofLem_LocCons_Aux_LocSVMsSupNormEmpTheo_AbschSup}
		\normSup{\fnitheo} \le \normSup{\kni} \cdot \norm{\Hni}{\fnitheo} \le \normSup{\kni} \cdot \risk[\Pni](0)^{1/2} \cdot \lbni^{-1/2} \le c_{p,\loss,\rho,\kappa} \cdot \lbni^{-1/2} 
	\end{align}
	with $c_{p,\loss,\rho,\kappa}\in(0,\infty)$ denoting a constant depending only on $p$, \loss, $\rho$ and $\kappa$, but not on $\lbni$.
	
	Assume now without loss of generality that $\dni>0$ (which by \eqref{eq:Lem_LocCons_Aux_LocSVMsSupNormEmpTheo_BedingungLambda} has to be satisfied for $n$ sufficiently large), \ie that \fniemp is indeed an empirical SVM and not just defined as the zero function. We know from \citet[Corollary 5.11]{steinwart2008} that there exists a function $h_{n,i}\colon\Xni\times\Y\to\R$ such that 
	\begin{align}\label{eq:ProofLem_LocCons_Aux_LocSVMsSupNormEmpTheo_DiffHNorm}
		\norm{\Hni}{\fniemp-\fnitheo} \le \frac{1}{\lbni} \cdot \norm{\Hni}{\ew[\DVertni]{h_{n,i}\Phi_{n,i}}-\ew[\Pni]{h_{n,i}\Phi_{n,i}}}\,
	\end{align}
	and, for $s:=p/(p-1)$,
	\begin{align}\label{eq:ProofLem_LocCons_Aux_LocSVMsSupNormEmpTheo_hAbsch}
		\norm{\L[\Pni]{s}}{h_{n,i}} &\le 8^p \cdot c_\loss \cdot \left(1+ |\Pni|_p^{p-1} + \normSup{\fnitheo}^{p-1} \right)\notag\\ 
		&\le 8^p \cdot  c_\loss \cdot \left(1+ \rho^{p-1} + c_{p,\loss,\rho,\kappa}^{p-1} \cdot \lbni^{-(p-1)/2} \right)\notag\\
		&\le \tilde{c}_{p,\loss,\rho,\kappa} \cdot \lbni^{-(p-1)/2} \,,
	\end{align}
	where we employed \eqref{eq:ProofLem_LocCons_Aux_LocSVMsSupNormEmpTheo_AbschSup} in the second and $\lbni\le C$ in the third step, and where $c_\loss\in (0,\infty)$ and $\tilde{c}_{p,\loss,\rho,\kappa} \in (0,\infty)$ denote constants depending only on $\loss$ respectively $p$, \loss, $\rho$ and $\kappa$.
	
	Assume without loss of generality that $p_2^*\le 1$ if $p=1$. Then, we can apply \citet[Lemma 9.2]{steinwart2008} with $q:=p/(p-1)$ if $p>1$ and $q:=2/p_2^*$ if $p=1$, which leads to $q^*:=\min\{1/2,1-1/q\}=\min\{1/2,1/p\} = (p+1)/(2p_1^*)$, to the functions $h_{n,i}\Phi_{n,i}$, $n\in\N$: First of all, with the help of \eqref{eq:ProofLem_LocCons_Aux_LocSVMsSupNormEmpTheo_hAbsch} we obtain
	\begin{align*}
		\norm{q}{h_{n,i}\Phi_{n,i}} :=&\, \left(\ew[\Pni]{\norm{\Hni}{h_{n,i}\Phi_{n,i}}^q}\right)^{1/q}\\ 
		\le&\, \normSup{\kni} \cdot \norm{\L[\Pni]{q}}{h_{n,i}} \le \kappa \cdot \tilde{c}_{p,\loss,\rho,\kappa} \cdot \lbni^{-(p-1)/2} < \infty \,,
	\end{align*}
	where we employed that, for all $(x,y)\in\Xni\times\Y$,
	\begin{align*}
		\norm{\Hni}{h_{n,i}(x,y)\Phi_{n,i}(x)}^q &= |h_{n,i}(x,y)|^q \cdot \norm{\Hni}{\Phi_{n,i}(x)}^q\\
		&= |h_{n,i}(x,y)|^q \cdot \kni(x,x)^{q/2} \le |h_{n,i}(x,y)|^q \normSup{\kni}^q
	\end{align*}
	by the reproducing property \citep[\vgl for example][Definition~2.9]{schoelkopf2002}. Hence, we obtain for all $\eps>0$, by combining this Lemma 9.2 with \eqref{eq:ProofLem_LocCons_Aux_LocSVMsSupNormEmpTheo_DiffHNorm},
	\begin{align*}
		&\Pni^{\dni}\left(\Dni\in(\Xni\times\Y)^{\dni} : \norm{\Hni}{\fniemp-\fnitheo} \ge \frac{\eps}{\kappa}\right) \\
		&\le \Pni^{\dni}\left(\Dni\in(\Xni\times\Y)^{\dni} : \norm{\Hni}{\ew[\DVertni]{h_{n,i}\Phi{n,i}}-\ew[\Pni]{h_{n,i}\Phi_{n,i}}} \ge \frac{\lbni\eps}{\kappa}\right)\\
		&\le c_q \cdot \left(\frac{\kappa\norm{q}{h_{n,i}\Phi_{n,i}}}{\lbni\eps \dni^{q^*}}\right)^q \le c_{q,p,\loss,\rho,\kappa} \cdot \left(\frac{1}{\lbni^{(p+1)/2}\eps \dni^{q^*}}\right)^q
	\end{align*}
	with $c_q\in(0,\infty)$ and $c_{q,p,\loss,\P,\k}\in(0,\infty)$ denoting constants depending only on $q$ (which means only on $p$ in the case $p>1$) respectively $q$, $p$, \loss, $\rho$ and $\kappa$. 
	
	With this, we can now return to investigating the whole global predictors with the help of \eqref{eq:ProofLem_LocCons_Aux_LocSVMsSupNormEmpTheo_Aufteilung}: For all $\eps>0$ and $n\in\N$, we have
	\begin{align}\label{eq:ProofLem_LocCons_Aux_LocSVMsSupNormEmpTheo_WktAbsch}
		&\P^n\left(\Dn\in(\X\times\Y)^n : \normLinftypxvar{\fnempext-\fntheoext}\ge \eps \right.\notag\\ 
		&\hspace*{8cm}\left.\Big|\, |\Dni[1]|=\dni[1],\dots,|\Dni[m_n]|=\dni[m_n] \right) \notag\\
		&\le \P^n\left(\Dn\in(\X\times\Y)^n : \max_{i\in\IndexPpositive} \norm{\Hni}{\fniemp-\fnitheo} \ge \frac{\eps}{\kappa}\right.\notag\\ 
		&\hspace*{8cm}\left.\Big|\, |\Dni[1]|=\dni[1],\dots,|\Dni[m_n]|=\dni[m_n] \right) \notag\\
		&\le \sum_{i\in\IndexPpositive} \Pni^{\dni}\left(\Dni\in(\Xni\times\Y)^{\dni} : \norm{\Hni}{\fniemp-\fnitheo} \ge \frac{\eps}{\kappa}\right) \notag\\
		&\le  c_{q,p,\loss,\rho,\kappa} \cdot \AnzPpositive \cdot \max_{i\in\IndexPpositive}\left(\frac{1}{\lbni^{(p+1)/2}\eps \dni^{q^*}}\right)^q\,,
	\end{align}
	and it remains to further investigate the right hand side:
	
	If $p>1$, we obtain $(qq^*)^{-1} = ((p-1)/p) \cdot \max\{2,p\} = p_2^*$. If $p=1$, we analogously obtain $(qq^*)^{-1} = (p_2^*/2) \cdot 2 = p_2^*$. Thus, we have
	\begin{align*}
		\AnzPpositive \cdot \max_{i\in\IndexPpositive} \left(\frac{1}{\lbni^{(p+1)/2}\dni^{q^*}}\right)^q &= \max_{i\in\IndexPpositive} \left(\frac{\AnzPpositive^{1/(qq^*)}}{\lbni^{(p+1)/(2q^*)}\dni}\right)^{qq^*}\\ 
		&= \max_{i\in\IndexPpositive} \left(\frac{\AnzPpositive^{p_2^*}}{\lbni^{p_1^*}\dni}\right)^{qq^*} \qquad\to 0\,,\qquad n\to\infty\,,
	\end{align*}
	by assumption. Hence, the whole right hand side of \eqref{eq:ProofLem_LocCons_Aux_LocSVMsSupNormEmpTheo_WktAbsch} converges to 0, which completes the proof.
\end{proof}

\begin{mylem}\label{Lem:LocCons_Aux_LocSVMsRiskTheoBayes}
	Let \Cref{Ann:LocCons_Pre_AllgAnn} be satisfied. Let $\loss\colon\YR\to[0,\infty)$ be a convex, distance-based loss function of upper growth type $p\in[1,\infty)$. Let $\ell\in\N$ and let, for $j=1,\dots,\ell$, $\kjbold:=(\kjr)_{r\in\Ij}$ be a family of measurable kernels of type $\betajbold:=(\betajr)_{r\in\Ij}$ on \X with RKHSs $(\Hjr)_{r\in\Ij}$ such that $\Hjr[0]\subseteq\Lppxvar$ dense. Assume that $|\P|_p<\infty$. Let \fntheoext, $n\in\N$, be defined as in \eqref{eq:LocCons_Loc_DefGlobalPredictor} such that the underlying regionalizations and weight functions satisfy \textbf{(R1)}, \textbf{(R2)}, \textbf{(W1)}, \textbf{(W2)} and \textbf{(W3)}, and such that 
	\begin{align*}
		\kni\in\{\kjr\einschraenkung_{\Xni\times\Xni} : j\in\{1,\dots,\ell\}, r\in\Ij\}
	\end{align*}
	for all $n\in\N$ and $i\in\{1,\dots,m_n\}$. If the regularization parameters satisfy $\lbni>0$ for all $n\in\N$ and $i\in\{1,\dots,m_n\}$ as well as $\max_{i\in\IndexPpositive}\betani^2\lbni\to 0$ as $n\to\infty$, then
	\begin{align*}
		\limn \risk(\fntheoext) = \riskbayes\,.
	\end{align*}
\end{mylem}

\begin{proof}
	Define the inner risk \innerriskbed as
	\begin{align*}
		\innerriskbed(t) := \int_\Y \loss(y,t) \diff\Pbed[y]{x}\qquad \forall \, x\in\X, t\in\R\,
	\end{align*}
	and denote by
	\begin{align*}
		\innerriskbedbayes := \inf_{t\in\R} \innerriskbed(t) \qquad \forall \, x\in\X
	\end{align*}
	the minimal inner risk at $x$. We will use these in order to split the risk of a given function (and the Bayes risk) into an outer integral with respect to \Pxvar and the inner risk.
	
	First, we however show that all risks appearing in the assertion are finite: \citet[Lemma~2.38(i)]{steinwart2008} yields $\risk(0)<\infty$ as well as $\risk[\loss,\Pni](0)<\infty$ for all $n\in\N$ and $i\in\IndexPpositive$ (with the latter holding true because $|\Pni|_p<\infty$ by \Cref{Bem:LocCons_Cons_MomentenAlternativen}). Since $\riskbayes\le\risk(0)$ by definition, we obtain the finiteness of \riskbayes. Furthermore,
	\begin{align*}
		\risk(\fntheoext) &= \int_{\XY} \loss(y,\fntheoext(x))\diff\P(x,y)\\
		&\le \int_{\XY} \sum_{i=1}^{m_n} \wni(x) \cdot \loss(y,\fnitheoextdach(x))\diff\P(x,y)\\
		&\le \sum_{i=1}^{m_n} \int_{\Xni\times\Y} \loss(y,\fnitheoext(x))\diff\P(x,y)\\
		&= \sum_{i\in\IndexPpositive} \P(\Xni\times\Y) \cdot \risk[\loss,\Pni](\fnitheoext)\,,
	\end{align*}
	where we applied \textbf{(W1)}, \textbf{(W2)} and the convexity of \loss in the second and its non-negativity as well as \textbf{(W1)} and \textbf{(W3)} in the third step. In the last step, we employed that $\Xni\times\Y$ is a \P-zero set for $i\notin\IndexPpositive$, leading to the according \P-integrals being 0.
	Since $\risk[\loss,\Pni](\fnitheoext)\le\risk[\loss,\Pni](0)$ for all $i\in\IndexPpositive$ by the definition of \fnitheoext, and since we already saw that $\risk[\loss,\Pni](0)<\infty$, the finiteness of $\risk(\fntheoext)$ follows for all $n\in\N$.
	
	With this, we can now write
	\begin{align}\label{eq:ProofLem_LocCons_Aux_LocSVMsRiskTheoBayes_RiskDiff}
		&\risk(\fntheoext) - \riskbayes \notag\\
		&= \int_\X \left(\innerriskbed(\fntheoext(x)) - \innerriskbedbayes\right) \diff\Pxvar(x)\notag\\
		&\le \int_\X \sum_{i=1}^{m_n}\wni(x) \cdot \left(\innerriskbed(\fnitheoextdach(x)) - \innerriskbedbayes\right) \diff\Pxvar(x)\notag\\
		&\le \sum_{i=1}^{m_n} \int_{\Xni} \left(\innerriskbed(\fnitheoext(x)) - \innerriskbedbayes\right) \diff\Pxvar(x)\notag\\
		&= \sum_{i\in\IndexPpositive} \left( \P(\Xni\times\Y) \cdot \risk[\loss,\Pni](\fnitheoext) - \int_{\Xni} \innerriskbedbayes \diff\Pxvar(x) \right) \,,
	\end{align}
	where we applied \citet[Lemma 3.4]{steinwart2008} in the first, \textbf{(W1)}, \textbf{(W2)} and the convexity of \loss in the second, and \textbf{(W1)}, \textbf{(W3)} and $\innerriskbed(\fnitheoext) - \innerriskbedbayes \ge 0$ for all $x\in\X$ (by the definition of \innerriskbedbayes) in the third step. In the final step, we once more used that $\P(\Xni\times\Y)=0$ for $i\notin\IndexPpositive$.

	If we define $\lbntilde:=\max_{i\in\IndexPpositive}\betani^2\lbni$ as well as $\knitilde\in\{\kjr:j\in\{1,\dots,\ell\}, r\in\Ij\}$ such that $\knitilde\einschraenkung_{\Xni\times\Xni} = \kni$ and analogously $\kninulltilde\in\{\kjr[0]:j\in\{1,\dots,\ell\}\}$ such that $\kninulltilde\einschraenkung_{\Xni\times\Xni} = \kninull$, we can further analyze the right hand side of \eqref{eq:ProofLem_LocCons_Aux_LocSVMsRiskTheoBayes_RiskDiff} by noting that, for all $n\in\N$ and $i\in\IndexPpositive$,
	\begin{align*}
		&\risk[\loss,\Pni](\fnitheoext)\\
		&\le \risk[\loss,\Pni](\fnitheoext) + \lbni \cdot \norm{\Hni}{\fnitheoext}^2\\
		&\le \risk[\loss,\Pni](\fninull) + \lbni \cdot \norm{\Hni}{\fninull}^2\\
		&\le \risk[\loss,\Pni](\fninull) + \betani^2 \cdot \lbni \cdot \norm{\Hninull}{\fninull}^2\\
		&\le \risk[\loss,\Pni](\fnitildenull\einschraenkung_{\Xni}) + \betani^2 \cdot \lbni \cdot \norm{\Hninull}{\fnitildenull\einschraenkung_{\Xni}}^2 \\
		&\le \risk[\loss,\Pni](\fnitildenull\einschraenkung_{\Xni}) + \lbntilde \cdot \norm{\Hninull}{\fnitildenull\einschraenkung_{\Xni}}^2 \\
		&\le \risk[\loss,\Pni](\fnitildenull\einschraenkung_{\Xni}) + \lbntilde \cdot \norm{\Hninulltilde}{\fnitildenull}^2\,.
	\end{align*}
	Here, we employed the definition of \fnitheoext respectively \fninull as the minimizers of the respective regularized risks \citep[combined with the fact that $\fninull\in\Hninull\subseteq\Hni$ and that $\fnitildenull\einschraenkung_{\Xni}\in\Hninull$ by][Theorem 6]{berlinet2004} in the second and in the fourth step, and again \citet[Theorem 6]{berlinet2004} in the last step. Furthermore, the third step holds true because 
	\begin{align*}
		\norm{\Hni}{f} &= \min_{\substack{g\in\Hnitilde:\\ g\einschraenkungklein_{\Xni}=f}} \norm{\Hnitilde}{g} \le \min_{\substack{g\in\Hninulltilde:\\ g\einschraenkungklein_{\Xni}=f}} \norm{\Hnitilde}{g} \le \betani \cdot \min_{\substack{g\in\Hninulltilde:\\ g\einschraenkungklein_{\Xni}=f}} \norm{\Hninulltilde}{g} = \betani \cdot \norm{\Hninull}{f}
	\end{align*}
	for all $f\in\Hninull$, where we once more applied \citet[Theorem 6]{berlinet2004} and that $\Hninulltilde\subseteq\Hnitilde$.
	
	Plugging this into the right hand side of \eqref{eq:ProofLem_LocCons_Aux_LocSVMsRiskTheoBayes_RiskDiff}, we obtain
	\begin{align}\label{eq:ProofLem_LocCons_Aux_LocSVMsRiskTheoBayes_RiskDiffTeilB}
		&\risk(\fntheoext) - \riskbayes\notag\\ 
		&\le \sum_{i\in\IndexPpositive} \left( \P(\Xni\times\Y) \cdot \left(\risk[\loss,\Pni](\fnitildenull\einschraenkung_{\Xni}) + \lbntilde \cdot \norm{\Hninulltilde}{\fnitildenull}^2\right)\right.\notag\\ 
		&\hspace*{9.5cm}\left.- \int_{\Xni} \innerriskbedbayes \diff\Pxvar(x) \right)\notag\\
		&= \sum_{i\in\IndexPpositive} \left( \P(\Xni\times\Y) \cdot \lbntilde\cdot \norm{\Hninulltilde}{\fnitildenull}^2\right.\notag\\ 
		&\hspace*{4.5cm}\left.+ \int_{\Xni} \left( \innerriskbed(\fnitildenull(x)) - \innerriskbedbayes \right) \diff\Pxvar(x) \right)\notag\\
		&\le \sum_{j=1}^{\ell} \sum_{i=1}^{m_n} \left( \P(\Xni\times\Y) \cdot \lbntilde\cdot \norm{\Hjr[0]}{\fnjnull}^2 \right.\notag\\ 
		&\hspace*{4.5cm}\left.+ \int_{\Xni} \left( \innerriskbed(\fnjnull(x)) - \innerriskbedbayes \right) \diff\Pxvar(x) \right)\notag\\
		&\le \sum_{j=1}^{\ell} \smax \cdot \left( \lbntilde \cdot \norm{\Hjr[0]}{\fnjnull}^2 + \risk(\fnjnull) -\riskbayes \right)\,,
	\end{align}
	with the third step holding true because of the summands being non-negative and the final step employing that, for all $j\in\{1,\dots,l\}$,
	\begin{align*}
		&\sum_{i=1}^{m_n} \int_{\Xni} \left( \innerriskbed(\fnjnull(x)) - \innerriskbedbayes \right) \diff\Pxvar(x)\\ 
		&= \int_\X \sum_{i=1}^{m_n} \Ind[\Xni](x) \cdot \left( \innerriskbed(\fnjnull(x)) - \innerriskbedbayes \right) \diff\Pxvar(x) \\
		&\le \smax \cdot \left( \risk(\fnjnull) -\riskbayes \right)
	\end{align*}
	by \textbf{(R2)}, and analogously $\sum_{i=1}^{m_n} \P(\Xni\times\Y) \le \smax$. 
	
	Now, by \citet[Lemma~2.38(i)]{steinwart2008}, \loss is a \P-integrable Nemitski loss of order $p$. Hence, for all $j\in\{1,\dots,l\}$, we know from \citet[Theorem 5.31]{steinwart2008} that
	\begin{align*}
		\riskoptH[\loss,\P,\H^{(j,0)}] := \inf_{f\in\H^{(j,0)}} \risk(f) = \riskbayes < \infty
	\end{align*} 
	and \citet[Lemma 5.15]{steinwart2008} then yields that
	\begin{align*}
		\limn \lbntilde \norm{\Hjr[0]}{\fnjnull}^2 + \risk(\fnjnull) - \riskbayes = 0
	\end{align*}
	because $\lbntilde\to 0$ as $n\to\infty$. Thus, the whole right hand side of \eqref{eq:ProofLem_LocCons_Aux_LocSVMsRiskTheoBayes_RiskDiffTeilB} converges to 0 as $n\to\infty$ and we obtain the assertion because $\risk(\fntheoext) - \riskbayes \ge 0$ by the definition of \riskbayes.
\end{proof}

\begin{mylem}\label{Lem:LocCons_Aux_LocSVMsRiskEmpTheo}
	Let \Cref{Ann:LocCons_Pre_AllgAnn} be satisfied. Let $\loss\colon\YR\to[0,\infty)$ be a convex, distance-based loss function of upper growth type $p\in[1,\infty)$. Assume that $|\P|_p<\infty$. Let \fntheoext and \fnempext, $n\in\N$, be defined as in \eqref{eq:LocCons_Loc_DefGlobalPredictor} and \eqref{eq:LocCons_Loc_DefGlobalPredictorEmp} such that the underlying regionalizations and weight functions satisfy \textbf{(R1)}, \textbf{(R3)}, \textbf{(W1)}, \textbf{(W2)}, \textbf{(W3)} and $\sup_{n\in\N,i\in\IndexPpositive} |\Pni|_p <\infty$. Assume that, for all $n\in\N$ and $i\in\{1,\dots,m_n\}$, \kni is a bounded and measurable kernel on \Xni with separable RKHS \Hni, such that\break $\sup_{n\in\N,i\in\IndexPpositive} \normSup{\kni} < \infty$. Define $p_1^*:=\max\{p+1,p(p+1)/2\}$ and $p_3^*:=\max\{p-1,p(p-1)/2\}$. Further choose $p_2^*:=\max\{2(p-1)/p,p-1\}$ if $p>1$ and $p_2^*\in (0,\infty)$ arbitrary if $p=1$. If the regularization parameters satisfy $\lbni\in(0,C)$ for all $n\in\N$ and $i\in\{1,\dots,m_n\}$ for some $C\in(0,\infty)$, as well as
	\begin{align}\label{eq:Lem_LocCons_Aux_LocSVMsRiskEmpTheo_BedingungLambda}
		\min_{i,j\in\IndexPpositive} \frac{\lbni[j]^{p_3^*}\lbni^{p_1^*}\dni}{\AnzPpositive^{p_2^*}} \to \infty
	\end{align}
	as $n\to\infty$, then
	\begin{align*}
		\limn \left|\risk(\fnempext)-\risk(\fntheoext)\right| = 0 \qquad \text{in probability $\P^\infty$.} 
	\end{align*}
	If additionally, the regionalizations \Xnbold, $n\in\N$, are partitions of \X, then it suffices if \eqref{eq:Lem_LocCons_Aux_LocSVMsRiskEmpTheo_BedingungLambda} is satisfied for $p_1^*:=\max\{2p,p^2\}$ and $p_3^*:=0$.
\end{mylem}

\begin{proof}
	Assume, for all $n\in\N$ and $i\in\IndexPpositive$, without loss of generality that $\dni>0$ (which by \eqref{eq:Lem_LocCons_Aux_LocSVMsSupNormEmpTheo_BedingungLambda} has to be satisfied for $n$ sufficiently large), such that the respective local empirical SVM \fniempext is indeed an empirical SVM and not just defined as the zero function. To shorten the notation, we denote $\fntheo:=\fntheoext$, $\fnemp:=\fnempext$, $\fnitheo:=\fnitheoext$ and $\fniemp:=\fniempext$ for all $n\in\N$ and $i\in\{1,\dots,m_n\}$, as well as $\kappa:=\sup_{n\in\N,i\in\IndexPpositive} \normSup{\kni}$, $\rho:=|\P|_p\vee\sup_{n\in\N,i\in\IndexPpositive} |\Pni|_p$ and $\lbntilde:=\min_{i\in\IndexPpositive}\lbni$ throughout this proof. Additionally, note that \Cref{Lem:LocCons_Aux_LocSVMsSupNormEmpTheo} is applicable in the situation of this lemma (in the base case as well as in the special case of the regionalizations being partitions of \X) as \eqref{eq:Lem_LocCons_Aux_LocSVMsRiskEmpTheo_BedingungLambda} in combination with $\lbni[j]\in(0,C)$ for all $n\in\N$ and $j\in\{1,\dots,m_n\}$ implies the validity of \eqref{eq:Lem_LocCons_Aux_LocSVMsSupNormEmpTheo_BedingungLambda}.
	
	We start by proving the main assertion before turning our attention to the special case of the regionalizations being partitions of \X afterwards.
	
	By applying \citet[Lemma~2.38(ii)]{steinwart2008} with $q:=p$, we obtain
	\begin{align}\label{eq:ProofLem_LocCons_Aux_LocSVMsRiskEmpTheo_RiskDiffAbsch}
		&|\risk(\fnemp)-\risk(\fntheo)|\notag\\ 
		&\le c_{p,\loss} \cdot \left( |\P|_p^{p-1} + \normLppxvar{\fntheo}^{p-1} + \normLppxvar{\fnemp}^{p-1} + 1 \right) \cdot \normLppxvar{\fnemp-\fntheo}\,,
	\end{align}
	where $c_{p,\loss}\in(0,\infty)$ denotes a constant only depending on $p$ and \loss.
	
	We can further analyze the right hand side of this inequality by noting that
	\begin{align*}
		\normLinftypxvar{\fntheo} &\le \max_{i\in\{1,\dots,m_n\}} \normLinftypxvar{\fnitheodach}\\ 
		&= \max_{i\in\IndexPpositive} \normLinftypnix{\fnitheo} \le \max_{i\in\IndexPpositive} c_{p,\loss,\rho,\kappa} \cdot \lbni^{-1/2}\,,
	\end{align*}
	with the first inequality following from \textbf{(W1)} and \textbf{(W2)}, similarly to \eqref{eq:ProofLem_LocCons_Aux_LocSVMsSupNormEmpTheo_Aufteilung}, 
	and the last one analogously to \eqref{eq:ProofLem_LocCons_Aux_LocSVMsSupNormEmpTheo_AbschSup}, with $c_{p,\loss,\rho,\kappa}\in(0,\infty)$ denoting a constant depending only on $p$, $\loss$, $\rho$ and $\kappa$.
	Hence, 
	\begin{align}\label{eq:ProofLem_LocCons_Aux_LocSVMsRiskEmpTheo_TheoSVMAbsch}
		\normLppxvar{\fntheo}^{p-1} &\le \normLinftypxvar{\fntheo}^{p-1} \le \max_{i\in\IndexPpositive} c_{p,\loss,\rho,\kappa}^{p-1} \cdot \lbni^{-(p-1)/2} = c_{p,\loss,\rho,\kappa}^{p-1} \cdot \lbntilde^{-(p-1)/2}\,.
	\end{align}
	Similarly, we obtain
	\begin{align}\label{eq:ProofLem_LocCons_Aux_LocSVMsRiskEmpTheo_EmpSVMAbsch}
		\normLppxvar{\fnemp}^{p-1} &\le \left(\normLppxvar{\fntheo} + \normLppxvar{\fnemp-\fntheo} \right)^{p-1}\notag\\
		&\le 2^{p-1} \cdot c_{p,\loss,\rho,\kappa}^{p-1} \cdot \lbntilde^{-(p-1)/2} + 2^{p-1} \cdot \normLppxvar{\fnemp-\fntheo}^{p-1}\,,
	\end{align}
	where we applied \eqref{eq:ProofLem_LocCons_Aux_LocSVMsRiskEmpTheo_TheoSVMAbsch} in the last step.
	
	Plugging \eqref{eq:ProofLem_LocCons_Aux_LocSVMsRiskEmpTheo_TheoSVMAbsch} and \eqref{eq:ProofLem_LocCons_Aux_LocSVMsRiskEmpTheo_EmpSVMAbsch} into \eqref{eq:ProofLem_LocCons_Aux_LocSVMsRiskEmpTheo_RiskDiffAbsch} then yields
	\begin{align*}
		&|\risk(\fnemp)-\risk(\fntheo)|\\
		&\le c_{p,\loss} \cdot \left( \rho^{p-1} + (2^{p-1}+1) \cdot c_{p,\loss,\rho,\kappa}^{p-1} \cdot \lbntilde^{-(p-1)/2} + 2^{p-1} \cdot \normLppxvar{\fnemp-\fntheo}^{p-1} + 1 \right)\notag\\ 
		&\hspace*{10cm} \cdot \normLppxvar{\fnemp-\fntheo}\\
		&= c_{p,\loss} \cdot \Big( \left( \rho^{p-1}  \lbntilde^{(p-1)/2} + (2^{p-1}+1) \cdot c_{p,\loss,\rho,\kappa}^{p-1} + \lbntilde^{(p-1)/2} \right) \notag\\ 
		&\hspace*{3cm} \cdot \lbntilde^{-(p-1)/2} \cdot \normLppxvar{\fnemp-\fntheo} + 2^{p-1} \cdot \normLppxvar{\fnemp-\fntheo}^p  \Big)\notag\\
		&\le \tilde{c}_{p,\loss,\rho,\kappa} \cdot \left(\lbntilde^{-(p-1)/2} \cdot \normLinftypxvar{\fnemp-\fntheo}  +  \normLinftypxvar{\fnemp-\fntheo}^p \right)\,,
	\end{align*}
	where we employed $\lbntilde\le C$ and $\normLppxvar{\fnemp-\fntheo}\le\normLinftypxvar{\fnemp-\fntheo}$ in the last step.
	
	We know from \Cref{Lem:LocCons_Aux_LocSVMsSupNormEmpTheo} that the second summand on the right hand side converges to 0 in probability as $n\to\infty$. Hence, we only need to further investigate the first summand. For this, we can proceed in exactly the same way as in the proof of \Cref{Lem:LocCons_Aux_LocSVMsSupNormEmpTheo} and only need to additionally consider the factor $\lbntilde^{-(p-1)/2}$. By doing this, we obtain for all $\eps>0$
	\begin{align}\label{eq:ProofLem_LocCons_Aux_LocSVMsRiskEmpTheo_WktAbsch}
		&\P^n\left(\Dn\in(\X\times\Y)^n : \lbntilde^{-(p-1)/2}\cdot\normLinftypxvar{\fnemp-\fntheo} \ge \eps\right.\notag\\
		&\hspace*{8cm}\left.\Big|\, |\Dni[1]|=\dni[1],\dots,|\Dni[m_n]|=\dni[m_n] \right) \notag\\
		&\le \P^n\left(\Dn\in(\X\times\Y)^n : \max_{i\in\IndexPpositive} \norm{\Hni}{\fniemp-\fnitheo} \ge \frac{\eps\lbntilde^{(p-1)/2}}{\kappa}\right.\notag\\ 
		&\hspace*{8cm} \Big|\, |\Dni[1]|=\dni[1],\dots,|\Dni[m_n]|=\dni[m_n] \Bigg) \notag\\
		&\le \sum_{i\in\IndexPpositive} \Pni^{\dni}\left(\Dni\in(\Xni\times\Y)^{\dni} : \norm{\Hni}{\fniemp-\fnitheo} \ge \frac{\eps\lbntilde^{(p-1)/2}}{\kappa}\right) \notag\\
		&\le  c_{q,p,\loss,\rho,\kappa} \cdot  \AnzPpositive \cdot \max_{i\in\IndexPpositive}\left(\frac{1}{\lbntilde^{(p-1)/2}\lbni^{(p+1)/2}\eps \dni^{q^*}}\right)^q\,,
	\end{align}
	analogously to \eqref{eq:ProofLem_LocCons_Aux_LocSVMsSupNormEmpTheo_WktAbsch}, with $c_{q,p,\loss,\rho,\kappa}\in(0,\infty)$ denoting a constant depending only on $q$, $p$, \loss, $\rho$ and $\kappa$. Here, as in the proof of \Cref{Lem:LocCons_Aux_LocSVMsSupNormEmpTheo}, $q:=p/(p-1)$ if $p>1$, $q:=2/p_2^*$ if $p=1$, and $q^*:=\min\{1/2,1-1/q\}=\min\{1/2,1/p\} = (p+1)/(2p_1^*) = (p-1)/(2p_3^*)$. 
	
	Because $(qq^*)^{-1} = p_2^*$ (\vgl proof of \Cref{Lem:LocCons_Aux_LocSVMsSupNormEmpTheo}), we furthermore obtain
	\begin{align*}
		\AnzPpositive \cdot \max_{i\in\IndexPpositive}\left(\frac{1}{\lbntilde^{(p-1)/2}\lbni^{(p+1)/2} \dni^{q^*}}\right)^q =  \max_{i\in\IndexPpositive}\left(\frac{\AnzPpositive^{p_2^*}}{\lbntilde^{p_3^*}\lbni^{p_1^*} \dni}\right)^{qq^*} \to 0\,,\qquad n\to\infty\,,
	\end{align*}
	by assumption. Hence, the whole right hand side of \eqref{eq:ProofLem_LocCons_Aux_LocSVMsRiskEmpTheo_WktAbsch} converges to 0, which yields the main assertion.
	
	As for the special case of the regionalizations being partitions of \X: If \Xnbold is a partition of \X, then the conditions \textbf{(W2)} and \textbf{(W3)} imply that $\wni=\Ind[\Xni]$ for all $i\in\{1,\dots,m_n\}$. Hence, we obtain
	\begin{align}\label{eq:ProofLem_LocCons_Aux_LocSVMsRiskEmpTheo_Falli}
		&|\risk(\fnemp)-\risk(\fntheo)|\notag\\ 
		&= \left| \int_{\XY} \loss\left(y,\sum_{i=1}^{m_n}\Ind[\Xni](x)\fniempdach(x)\right) \diff\P(x,y)\right.\notag\\ 
		&\hspace*{6cm}\left.- \int_{\XY} \loss\left(y,\sum_{i=1}^{m_n}\Ind[\Xni](x)\fnitheodach(x)\right) \diff\P(x,y) \right|\notag\\
		&= \left| \sum_{i=1}^{m_n} \left( \int_{\Xni\times\Y} \loss(y,\fniemp(x)) \diff\P(x,y) - \int_{\Xni\times\Y} \loss(y,\fnitheo(x)) \diff\P(x,y) \right) \right|\notag\\
		&\le \sum_{i\in\IndexPpositive} \P(\Xni\times\Y) \cdot \left| \risk[\loss,\Pni](\fniemp) - \risk[\loss,\Pni](\fnitheo) \right|\notag\\
		&\le \max_{i\in\IndexPpositive} \left| \risk[\loss,\Pni](\fniemp) - \risk[\loss,\Pni](\fnitheo) \right|
	\end{align}
	in this case. In the third step, we applied that $\Xni\times\Y$ is a \P-zero set for all $i\notin\IndexPpositive$, leading to the according \P-integrals being 0.
	
	The argument of the maximum on the right hand side of \eqref{eq:ProofLem_LocCons_Aux_LocSVMsRiskEmpTheo_Falli} can, for each $i\in\IndexPpositive$, be examined in the same way as we previously examined the difference on the left hand side for proving the main assertion. A difference appears in \eqref{eq:ProofLem_LocCons_Aux_LocSVMsRiskEmpTheo_TheoSVMAbsch}, where we now have
	\begin{align*}
		\norm{\Lp(\Pnix)}{\fnitheo}^{p-1} \le \normLinftypnix{\fnitheo}^{p-1} \le c_{p,\loss,\rho,\kappa}^{p-1} \cdot \lbni^{-(p-1)/2}\,.
	\end{align*}
	That is, we can omit the final step of bounding this with the help of \lbntilde because we are now not interested in $\max_{i\in\IndexPpositive}\normLinftypnix{\fnitheo}$ but only in $\normLinftypnix{\fnitheo}$ for a specific $i$.
	
	By applying this to the subsequent steps of our proof, we obtain
	\begin{align*}
		&|\risk(\fnemp)-\risk(\fntheo)|\\
		&\le \max_{i\in\IndexPpositive} \left| \risk[\loss,\Pni](\fniemp) - \risk[\loss,\Pni](\fnitheo) \right|\\
		&\le \tilde{c}_{p,\loss,\rho,\kappa} \cdot \max_{i\in\IndexPpositive} \left( \lbni^{-(p-1)/2} \cdot \normLinftypnix{\fniemp-\fnitheo} + \normLinftypnix{\fniemp-\fnitheo}^p \right)\\
		&\le \tilde{c}_{p,\loss,\rho,\kappa} \cdot \left( \max_{i\in\IndexPpositive} \left( \lbni^{-(p-1)/2} \cdot \normLinftypnix{\fniemp-\fnitheo} \right) + \normLinftypxvar{\fnemp-\fntheo}^p \right)\,,
	\end{align*}
	where the second summand on the right hand side converges to 0 in probability by \Cref{Lem:LocCons_Aux_LocSVMsSupNormEmpTheo}.
	
	As for the first summand, we can derive
	\begin{align*}
		&\P^n\Bigg(\Dn\in(\X\times\Y)^n : \max_{i\in\IndexPpositive} \left( \lbni^{-(p-1)/2} \cdot \normLinftypnix{\fniemp-\fnitheo} \right) \ge \eps\notag\\ 
		&\hspace*{8cm} \Big|\, |\Dni[1]|=\dni[1],\dots,|\Dni[m_n]|=\dni[m_n] \Bigg) \notag\\
		&\le  c_{q,p,\loss,\rho,\kappa} \cdot  \AnzPpositive \cdot \max_{i\in\IndexPpositive}\left(\frac{1}{\lbni^{p}\eps \dni^{q^*}}\right)^q\,,
	\end{align*}
	analogously to \eqref{eq:ProofLem_LocCons_Aux_LocSVMsRiskEmpTheo_WktAbsch}. Finally, we obtain convergence to 0 of the right hand side, and thus the assertion, because 
	\begin{align*}
		\AnzPpositive \cdot \max_{i\in\IndexPpositive}\left(\frac{1}{\lbni^{p} \dni^{q^*}}\right)^q =  \max_{i\in\IndexPpositive}\left(\frac{\AnzPpositive^{p_2^*}}{\lbni^{p_1^*} \dni}\right)^{qq^*} \to 0\,,\qquad n\to\infty\,,
	\end{align*}
	by assumption, where we applied that $p/q^*=p_1^*$ since $p_1^*=\max\{2p,p^2\}$ now.
\end{proof}

\section{Proofs}

\begin{proof}[Proof of \Cref{Thm:LocCons_Cons_LpCons}]
	We can split up the difference, which we wish to investigate, as
	\begin{align}\label{eq:ProofThm_LocCons_Cons_LpCons_Aufteilung}
		&\normLppxvar{\fnempext-\fbayes} \notag\\
		&\le \normLppxvar{\fnempext-\fntheoext} + \normLppxvar{\fntheoext-\fbayes}\,.
	\end{align}
	Because $\normLppxvar{\fnempext-\fntheoext}\le\normLinftypxvar{\fnempext-\fntheoext}$, we know from \Cref{Lem:LocCons_Aux_LocSVMsSupNormEmpTheo} that the first summand on the right hand side converges to 0 in probability as $n\to\infty$. 
	
	Thus, only the second summand remains to be examined: From \Cref{Lem:LocCons_Aux_LocSVMsRiskTheoBayes}, we obtain 
	\begin{align*}
		\limn\risk(\fntheoext) = \riskbayes\,.
	\end{align*}
	
	We further know for all $n\in\N$ that $\fntheoext\in\Lppxvar$ because
	\begin{align*}
		&\normLppxvar{\fntheoext} \le \normLinftypxvar{\fntheoext} \le \max_{i\in\{1,\dots,m_n\}} \normLinftypxvar{\fnitheoextdach}\\ 
		&\le \max_{i\in\IndexPpositive} \normLinftypnix{\fnitheoext} \le \max_{i\in\IndexPpositive} \normSup{\kni} \norm{\Hni}{\fnitheoext} < \infty 
	\end{align*}
	by \textbf{(W1)}, \textbf{(W2)} and \citet[Lemma~4.23]{steinwart2008}, similarly to \eqref{eq:ProofLem_LocCons_Aux_LocSVMsSupNormEmpTheo_Aufteilung}. Employing \citet[Theorem 3.2 and Remark 3.3]{koehler2023arxiv} then yields convergence to 0 (as $n\to\infty$) of the second summand on the right hand side of \eqref{eq:ProofThm_LocCons_Cons_LpCons_Aufteilung}, which completes the proof.
\end{proof}

\begin{proof}[Proof of \Cref{Thm:LocCons_Cons_RiskCons}]	
	We start by proving the main assertion and the special case (i): We can split up the difference, which we wish to investigate, as
	\begin{align}\label{eq:ProofThm_LocCons_Cons_RiskCons_Aufteilung}
		&|\risk(\fnempext)-\riskbayes|\notag\\ 
		&\le |\risk(\fnempext)-\risk(\fntheoext)| + |\risk(\fntheoext)-\riskbayes|\,. 
	\end{align}
	The assertions then follow directly by applying \Cref{Lem:LocCons_Aux_LocSVMsRiskEmpTheo} to the first and \Cref{Lem:LocCons_Aux_LocSVMsRiskTheoBayes} to the second summand on the right hand side.
	
	As for the special case (ii): If \fbayes is \Pxvar-\fs unique, the assertion follows directly from \Cref{Thm:LocCons_Cons_LpCons} and \citet[Theorem~3.4]{koehler2023arxiv}, which is applicable because $\fbayes\in\Lppxvar$ \citep[\vgl][Remark~3.3]{koehler2023arxiv} and $\fnempext\in\Lppxvar$ for all $n\in\N$ (\vgl proof of \Cref{Thm:LocCons_Cons_LpCons}).
\end{proof}

\end{document}